\documentclass[twoside,11pt]{article}

\usepackage[preprint]{custom_jmlr2e}

\usepackage{amsmath}
\usepackage{pdfpages}

\newcommand{\argmax}{\mbox{\,\rm arg\,max}}

\newcommand{\xins}{x \in S}

\jmlrheading{volume}{2019}{pages}{date submitted}{date published}{paper id}{Wesley Cowan, Michael N. Katehakis, and Daniel Pirutinsky}

\ShortHeadings{Accelerating Computation of UCB and other Indices}{Cowan, Katehakis, and Pirutinsky}
\firstpageno{1}

\title{Accelerating the Computation of UCB and Related Indices for Reinforcement Learning}

\author{\name Wesley Cowan \email cwcowan@cs.rutgers.edu \\
    \addr Department of Computer Science \\
    Rutgers University \\
    110 Frelinghuysen Road, Piscataway, NJ 08854, USA
    \AND
      \name Michael N. Katehakis \email mnk@rutgers.edu \\
    \addr Department of Management Science and Information Systems\\
    Rutgers University\\
    100 Rockafeller Road, Piscataway, NJ 08854, USA
    \AND
    \name Daniel Pirutinsky \email dp771@scarletmail.rutgers.edu \\
    \addr Department of Management Science and Information Systems\\
    Rutgers University\\
    100 Rockafeller Road, Piscataway, NJ 08854, USA}

\editor{...}

\begin{document}
    
\maketitle

\begin{abstract}%
In this paper we derive an efficient method for computing the indices associated with an asymptotically optimal upper confidence bound algorithm (MDP-UCB) of \citet{burnetas1997optimal} that only requires solving a system of two non-linear equations with two unknowns, irrespective of the cardinality of the state space of the Markovian decision process (MDP). In addition, we develop a similar acceleration for computing the indices for the MDP-Deterministic Minimum Empirical Divergence (MDP-DMED) algorithm developed in \citet{cowan2019reinforcement}, based on ideas from \citet{honda2011asymptotically}, that involves solving a single equation of one variable. We provide experimental results demonstrating the computational time savings and regret performance of these algorithms. In these comparison we also consider the Optimistic Linear Programming (OLP) algorithm \citep{tewari2008optimistic} and a method based on Posterior sampling (MDP-PS).
\end{abstract}


\section{Introduction}
The practical use of the asymptotically optimal UCB algorithm (MDP-UCB) of \citet{burnetas1997optimal} has been hindered \citep{tewari2008optimistic,auer2007logarithmic} by the computational burden of the upper confidence bound indices c.f. Eq. \eqref{eq:UCB-index}, that involves the solution of a non-linear constrained optimization problem of dimension equal to the cardinality of the state space of the Markovian decision process (MDP) under consideration. In this paper we derive an efficient computational method that only requires solving a system of two non-linear equations with two unknowns, irrespective of the cardinality of the state space of the MDP. In addition, we develop a similar acceleration for  computing the indices for the MDP-Deterministic Minimum Empirical Divergence (MDP-DMED) developed in \citet{cowan2019reinforcement}, that involves solving a single equation of one variable. In Section \ref{sec:computation} we present these computationally efficient formulations and provide experimental results demonstrating the computational time savings. 

\subsection{Related Work}
Many modern ideas of reinforcement learning originate in work done for the multi-armed bandit problem c.f. \citet{gittins1979bandit, gittins2011multi}, \citet{auer2002finite}, \citet{whittle1980multi}, \citet{weber1992gittins},  \citet{sonin2016continue}, \citet{mahajan2008multi}, \citet{katehakis1987multi, katehakis1996finite, katehakis1986computing}.

In addition to the papers upon which the algorithms here are explicitly based, there are many other approaches for adaptively learning MDPs while minimizing expected regret. \citet{jaksch2010optimal} propose an algorithm, UCRL2, a variant of the UCRL algorithm of \citet{auer2007logarithmic}, that achieves logarithmic regret asymptotically, as well as uniformly over time. UCRL2, defines a set of plausible MDPs and chooses a near-optimal policy for an optimistic version of the MDP through so called ``extended value iteration''. This approach, while similarly optimistic in flavor, is sufficiently different than the algorithms presented here that we will not be comparing them directly. The algorithms in this paper act upon the estimated transition probabilities of actions for only our current state, for a fixed estimated MDP. Specifically, MDP-UCB and OLP inflate the right hand side of the optimality equations by perturbing the estimated transition probabilities for actions in the current state. MDP-DMED estimates the rates at which actions should be taken by exploring nearby plausible transition probabilities for actions in the current state. Finally, MDP-PS obtains posterior sampled estimates, again, only for, the transition probabilities for actions in the current state.

Recently, \cite{efroni2019tight} show that model-based algorithms (which all the algorithms discussed here are), that use 1-step planning can achieve the same regret performance as algorithms that perform full-planning. This allows for a significant decrease in the computational complexity of the algorithms. In particular they propose UCRL2-GP, which uses a greedy policy instead of solving the MDP as in UCRL2, at the beginning of each episode. They find that this policy matches UCRL2 in terms of regret (up to constant and logarithmic factors), while benefiting from decreased computational complexity. The setting under consideration however, is a finite horizon MDP and the regret bounds are in PAC terms \citep{dann2017unifying} and optimal minimax \citep{osband2016lower}. Further analysis is required to transfer these results to the setting of this paper. Namely, an infinite horizon MDP with bounds on the asymptotic growth rate of the expected regret. A fruitful direction of study would be to examine the relationship between UCRL2-GP, UCRL2, and the algorithms presented here, more closely, paying particular attention to the varying dependencies on the dimensionality of the state space.

\citet{osband2017why} analyze and compare the expected regret and computational complexity of PS-type algorithms (PSRL therein) versus UCB-type (OFU therein) algorithms, in the setting of finite horizon MDPs. The PSRL algorithm presented there is similar to MDP-PS here. However, their optimistic inflation or stochastic optimism is done across the MDP as a whole, either over plausible MDPs in the case of OFU, or for a fixed MDP in the PSRL case. By contrast, in this paper we present non-episodic versions where the inflations are done only for the actions of our current state for a fixed estimated MDP. They also argue therein that any OFU approach which matches PSRL in regret performance will likely result in a computationally intractable optimization problem. Through that lens, the main result of this paper, proving a computationally tractable version of the optimization problem shows that actually a provably asymptotically optimal UCB approach \textbf{can} compete with a PS approach both in terms of regret performance as well as computational complexity. A more thorough analysis is required in order to determine what parts of our analysis here, with an undiscounted infinite horizon MDP, can carry over to the finite horizon MDP setting of \citet{osband2017why} and \citet{osband2016lower}.

As this is a fast growing area of research, there is a lot of recent work. A good resource for reinforcement learning problems and their potential solution methods is \citet{bertsekas2019reinforcement}. For a more bandit focused approach, \citet{lattimore2018bandit} has a nice overview of the current state of the art. Most directly relevant to this paper are Chapters 8, 10, and 38 therein. \citet{cesa-bianchi2006prediction} discuss online learning while minimizing regret for predicting individual sequences of various forms, with Chapter 6 (bandit related problems) therein being most relevant here. For other related early work we refer to \citet{Mandl74},
  \citet{borkarV82}, \citet{Agrawal88MDPs}, and  \citet{Agrawal88Mabs}.

\subsection{Paper Structure}
The paper is organized as follows. In Section \ref{sec:formulation} we formulate the problem under consideration first as a completely known MDP and then as an MDP with unknown transition laws. In Section \ref{sec:algorithms} we present four simple algorithms \footnote{A version of some of the algorithms and comparisons has appeared in a previous technical note \citet{cowan2019reinforcement}.}for adaptively optimizing the average reward in an unknown irreducible MDP. The first is the asymptotically optimal UCB algorithm (MDP-UCB) of \citet{burnetas1997optimal} that uses estimates for the MDP and choose actions by maximizing an inflation of the estimated right hand side of the average reward optimality equations. The second (MDP-DMED) is inspired by the DMED method for the multi-armed bandit problem developed in \citet{honda2010asymptotically, honda2011asymptotically} and estimates the optimal rates at which actions should be taken and attempts to take actions at that rate. The third is the Optimistic Linear Programming (OLP) algorithm \citep{tewari2008optimistic} which is based on MDP-UCB but instead of using the KL divergence to inflate the optimality equations, uses the $L_1$ norm. The fourth (MDP-PS) is based on ideas of greedy posterior sampling that go back to \citet{thompson1933likelihood} and similar to PSRL in \citet{osband2017why}. The main contribution of this paper is in Section \ref{sec:computation}, where we present the efficient formulations and demonstrate the computational time savings. Various computational challenges and simplifications are discussed, with the goal of making these algorithms practical for broader use. In Section \ref{sec:comparison} we compare the regret performance of these algorithms in numerical examples and discuss the relative advantages of each. While no proofs of optimality are presented, the results of numerical experiments are presented demonstrating the efficacy of these algorithms. Proof of optimality for these algorithms will be discussed in future works. 

\section{Formulation}\label{sec:formulation}
Reinforcement learning problems are commonly expressed in terms of a \textit{controllable}, \textit{probabilistic}, \textit{dynamic system}, where the dynamics must be learned over time. The classical model for this is that of a discrete time, finite state and action Markovian decision process (MDP). See for example, \citet{derman1970finite} and \citet{auer2007logarithmic}. In particular, learning is necessary when the underlying dynamics (the transition laws) are unknown, and must be learned by observing the effects of actions and transitions of the system over time.

A finite MDP  is specified by a quadruple $(S,A,R,P)$, where $S$ is a finite state space, $A=[ A(x) ]_{\xins}$ is the action space, with $A(x)$ being the (finite) set of admissible actions (or controls) in state $x$, $R=[r_{x,a}]_{\xins, a \in A(x)}$, is the expected reward structure and $P=[p^a_{x,y}]_{x,y \in S, a \in A(x)}$ is the transition law. Here $r_{x,a}$ and $p^a_{x,y}$ are respectively the one step expected reward and transition probability from state $x$ to state $y$ under action $a$. For extensions regarding state and action spaces and continuous time we refer to \citet{feinberg2016partially} and references therein.

When all elements of $(S, A, R, P)$ are known the model is said to be an MDP with {\sl complete information} (CI-MDP). In this case, optimal polices can be obtained via the appropriate version of Bellman's equations, given the prevailing  optimization criterion, state, action, time conditions and regularity assumptions; c.f. \citet{feinberg2016partially}, \citet{robbins1952some}.
When some of the elements of $(S,A,R,P)$ are unknown the model is said to be an MDP with incomplete or {\sl partial information} (PI-MDP). This is the primary model of interest for reinforcement learning, when some aspect of the dynamics must be learned through interaction with the system.

For the body of the paper, we consider the following partial information model: the transition probability vector $\underline{p}^a_x = [p^a_{x,y}]_{y \in S}$ is taken to be an element of parameter space
$$\Theta = \left\{ \underline{p} \in \mathbb{R}^{\lvert S \rvert} : \sum_{y \in S} p_y = 1 , \forall y \in S, p_{y} > 0\right\},$$
that is, the space of all $\lvert S \rvert$-dimensional probability vectors. 

The assertion of this parameter space deserves some unpacking. It is at first simply a theoretical convenience---it ensures that for any control policy, the resulting Markov chain is irreducible. It also represents a complete lack of prior knowledge about the transition dynamics of the MDP. Knowing that certain state-state transitions are impossible requires prior model specific knowledge (such knowing the rules of chess).  Learning based purely on finite observed data could never conclude that a given transition probability is zero. Thus, we assert a uniform Bayesian prior on the transition probabilities and therefore the likelihood associated with $p = 0$ is $0$. In this way, asserting this parameter space starting out represents a fairly agnostic initial view of the underlying learning problem. A possible future direction of study is to examine how to efficiently incorporate prior knowledge, for instance modifying the specified parameter space, into the learning process without compromising on the learning rate. \citet{killian2017robust} and \citet{doshi-velez2016hidden} discuss hidden parameterized transition models, for example, which leverage additional prior knowledge about the transition probability space.

In the body of this paper, we take this unknown transition law to be the only source of incomplete information about the underlying MDP; the reward structure $R = [ r_{x,a}]_{x \in S, a \in A(x)}$ is taken to be known (at least in expectation), and constant. Much of the discussed algorithms will generalize to the situation where the distribution of rewards must also be learned, but we reserve this for future work.

Under this model, we define a sequence of state valued random variables $X_1, X_2, X_3, \ldots$ representing the sequence of states of the MDP (taking $X_1 = x_1$ as a given initial state), and action valued random variables $A_1, A_2, \ldots$ as the action taken by the controller, action $A_t$ being taken at time $t$ when the MDP is in state $X_t$. It is convenient to define a control policy $\pi$ as a (potentially random) history dependent sequence of actions such that $\pi(t) = \pi(X_1,A_1,\ldots,X_{t-1},A_{t-1},X_t) = A_t \in A(X_t)$. We may then define the value of a policy as the total expected reward over a given horizon of action:
\begin{equation*}
V_\pi(T) = \mathbb{E}\left[ \sum_{t = 1}^T r_{X_t, A_t} \right].
\end{equation*}

Let $\Pi$ be the set of all feasible MDP policies $\pi$. We are interested in policies that maximize the expected reward from the MDP. In particular, policies that are capable of maximizing the expected reward irrespective of the initial uncertainty that exists about the underlying MDP dynamics (i.e., for all possible $P$ under consideration). It is convenient then to define $V(T) = \sup_{\pi \in \Pi} V_\pi(T)$. We may then define the ``regret'' as the expected loss due to ignorance of the underlying dynamics,
\begin{equation*}
R_\pi(T) = V(T) - V_\pi(T).
\end{equation*}
Note, $V, V_\pi, R_\pi$ all have an implicit dependence on $P$, through the dynamics of the states and effects of the actions.

We are interested in uniformly fast \citep{burnetas1997optimal} policies, policies $\pi$ that achieve $R_\pi(T) = O(\ln T)$ for \textit{all feasible transition laws} $P$. In this case, despite the controller's initial lack of knowledge about the underlying dynamics, she can be assured that her expected loss due to ignorance grows not only sub-linearly over time, but slower than any power of $T$. It is shown in \citet{burnetas1997optimal} that any uniformly fast policy has a strict lower bound of logarithmic asymptotic growth of regret, with the unknown transition law $P$ and reward structure $R$ only influencing the order coefficient, not the growth rate. Policies that achieve this lower bound are called asymptotically optimal c.f. \citet{burnetas1997optimal}.

As final notation, it is convenient to define the specific data available at any point in time, under a given (understood) policy $\pi$: let $T_x(t), T^a_x(t), T^a_{x,y}(t)$ be, respectively, the number of visits to state $x$, the uses of action $a$ in state $x$, and the transitions from $x$ to $y$ under action $a$, that are observed in the first $t$ rounds.

In the next subsection, we consider the case of the controller having complete information (the best possible case) and use this to motivate notation and machinery for the remainder of the paper. The body of the paper is devoted to presenting and discussing four computationally simple algorithm that are either provably asymptotically optimal, or at least appear to be. While no proofs of optimality are presented, the results of numerical experiments are presented demonstrating the efficacy of these algorithm. Proof of optimality for these algorithm will be discussed in future works.

\subsection{The Optimal Policy Under Complete Information}
Classical results \citep{burnetas1997optimal} show that there is a stationary, deterministic policy $\pi$ (each action depends only on the current state), that realizes the maximal long term average expected value. That is, a simple Markovian policy $\pi^*$ that realizes
\begin{equation*}
\lim_T \frac{ V_{\pi^*}(T) }{ T } = \phi^* =   \sup_{\pi \in \Pi} \liminf_T \frac{ V_\pi(T) }{T}.
\end{equation*}

We may characterize this optimal policy in terms of the solution for $\phi = \phi^*(A,P)$ and $\underline{v} = \underline{v}(A,P)$ of the following system of optimality equations:
\begin{equation}\label{eqn:dp}
\forall x \in S:\ 
\ \ \ \ \   \phi + v_x = \max_{a \in A(x)}\left( r_{x,a} + \sum_{y \in S} p^a_{x,y} v_y\right).
\end{equation}

Given the solution $\phi$ and vector $\underline{v}$ to the above equations, the asymptotically optimal policy $\pi^*$ can be characterized as, whenever in state $x \in S$, take any action $a$ which realizes the maximum in Eq. \eqref{eqn:dp}. We denote the set of such asymptotically optimal actions as $O(x,P)$. In general, $a^*(x,P)$ should be taken to denote an action $a^* \in O(x,P)$. Note, realizing this policy necessarily requires knowledge of $P$ and $R$, in order to solve the system of optimality equations.

The solution $\phi$ above represents the maximal long term average expected reward of an optimal policy. The vector $\underline{v}$, or more precisely, $v_x$ for any $x \in S$, represents in some sense the immediate value of being in state $x$ \textit{relative to the long term average expected reward}. The value $v_x$ essentially encapsulates the future opportunities for value available due to being in state $x$. 

\subsection{Optimal Policies Under Unknown Transition Laws}
The results of the previous section show that $V(T)$, the value of the optimal policy, goes approximately like $V(T) \approx \phi^* T$. We begin by characterizing the regret of any arbitrary policy $\pi$, comparing its value relative to this baseline. It will be convenient in what is to follow to define the following notation:
\begin{equation*}
L(x,a,\underline{p}, \underline{v}) = r_x(a) + \sum_{y \in S} p_y v_y.
\end{equation*}
The function $L$ represents the value of a given action in a given state, for a given transition vector---both the immediate reward, and the expected future value of whatever state the MDP transitions into. The value of an asymptotically optimal action for any state $x$ is thus given by $L^*(x,A,P) = L(x, a^*(x,P),\underline{p}_x^{a^*(x,P)}, \underline{v}(A,P))$. It can be  shown that the ``expected loss'' due to an asymptotically sub-optimal action, taking action $a \notin O(x,P)$ when the MDP is in state $x$, is in the limit given by
\begin{equation*}
\Delta(x,a,A,P)  = L^*(x,A,P) - L(x,a, \underline{p}^a_x, \underline{v}(A,P)).
\end{equation*}

In the general (partial or complete information) case, it is shown in \citet{burnetas1997optimal} that the regret of a given policy $\pi \in \Pi$ can be expressed asymptotically as
\begin{equation*}
R_\pi(T)   = V(T)  - V_\pi(T) = \sum_{x \in S} \sum_{a \notin O(x,P)} \mathbb{E}\left[ T^a_x(T) \right] \Delta(x,a,A,P) + O(1).
\end{equation*}

Note, the above formula justifies the description of $\Delta(x,a,A,P)$ as the ``average loss due to sub-optimal activation of $a$ in state $x$''. Additionally, from the above it is clear that in the case of complete information, when $P$ is known and therefore the asymptotically optimal actions are computable, the total regret at any time $T$ is bound by a constant. Any expected loss at time $T$ is due only to finite horizon effects. 

In general, for the unknown transition laws case, we have the following bound due to \citet{burnetas1997optimal}, for any uniformly fast policy $\pi$,
\begin{equation*}
\liminf_T \frac{ R_\pi(T) }{\ln T} \geq \sum_{x \in S} \sum_{a \notin O(x,P)} \frac{ \Delta(x,a,A,P) }{ \mathbf{K}_{x,a}(P) },
\end{equation*}
where $\mathbf{K}_{x,a}(P)$ represents the minimal Kullback-Leibler divergence between $\underline{p}^a_x$ and any $\underline{q} \in \Theta$ such that substituting $\underline{q}$ for $\underline{p}^a_xx$ in $P$ renders $a$ the unique optimal action for $x$. Recall, the Kullback-Leibler divergence is given by $\mathbf{I}(\underline{p}, \underline{q}) = \sum_{x \in S} p_x \ln( p_x / q_x )$. This is equivalent to stating that any sub-optimal action must be sampled at least at a minimum rate, in particular, for $a \notin O(x,P)$,
\begin{equation*}
\liminf_T \frac{ \mathbb{E}\left[ T^a_x(T) \right]  }{ \ln T } \geq \frac{ 1 }{  \mathbf{K}_{x,a}(P) }.
\end{equation*}
This can be interpreted in the following way: for a sub-optimal action, the ``closer'' the transition law is to an alternative transition law that would make it the best action, the more data we need to distinguish between the truth and this plausible alternative hypothesis, and therefore the more times we need to sample the action to distinguish the truth. Anything less than this ``base rate'', we risk convincing ourselves of a plausible, sub-optimal hypothesis and therefore incurring high regret when we act on that belief.

Policies that achieve this lower bound, for all $P$, are referred to as asymptotically optimal. Achieving this bound, or at least the desired logarithmic growth requires careful exploration of actions. In the next section, we present four algorithms to accomplish this.

\section{Algorithms for Optimal Exploration}\label{sec:algorithms}
Common RL algorithms solve the exploration/exploitation dilemma in the following way: most of the time, select an action (based on the current data) that seems best, otherwise select some other action. This alternative action selection is commonly done uniformly at random. As long as this is done infrequently, but not too infrequently, the optimal actions and policy will be discovered, potentially at the cost of high regret. Minimizing regret requires careful consideration of which alternative actions are worth taking at any given point in time. The following algorithms are methods for performing this selection; essentially, instead of blindly selecting from the available actions to explore, each algorithm evaluates the currently available data to determine which action is most worth exploring. Each accomplishes this through an exploration of the space of plausible transition hypotheses. 

The benefit of this is that through careful exploration, optimal (minimal) regret can be achieved. The cost however, is additional computation. The set of alternative transition laws is large and high dimensional, and can be difficult to work with. In Section \ref{sec:computation} we show several simplifications, however, that make this exploration practical.

\subsection{A UCB-Type Algorithm for MDPs Under Uncertain Transitions} \label{sec:UCB}
Classical upper confidence bound (UCB) decision algorithms (for instance as in multi-armed bandit problems, c.f. \citet{auer2010ucb}, \citet{burnetas1996optimal}, \citet{cowan2017normal}), approach the problem of exploration in the following way: in each round, given the current estimated transition law, we consider ``inflated'' estimates of the values of each actions, by finding the best (value-maximizing) plausible hypothesis within some confidence interval of the current estimated transition law. The more data that is available for an action, the more confidence there is in the current estimate, and the tighter the confidence interval becomes; the tighter the confidence interval becomes, the less exploration is necessary for that action. The algorithm we present here is a version of the MDP-UCB algorithm presented in \citet{burnetas1996optimal}.

At any time $t \geq 1$, let $x_t$ be the current (given) state of the MDP. We construct the following estimators:
\begin{itemize}
    \item Transition Probability Estimators: for each state $y$ and action $a \in A(x_t)$, construct $\hat{P}_t$ based on
    \begin{equation*}
    \hat{p}^a_{x_t,y} = \frac{ T^a_{x_t,y}(t) + 1 }{ T^a_{x_t}(t) + \lvert S \rvert }.
    \end{equation*}
    Note the biasing terms (the $1$ in the numerator, $\lvert S \rvert$ in the denominator). Including these, biases the estimated transition probabilities away from $0$, so that our estimates $\underline{p}_{x_t}^a$ will be in $\Theta$. Additionally, these guarantee that the above is in fact the maximum likelihood estimate for the transition probability, given the observed data and uniform priors.
    
    \item ``Good'' Action Sets: construct the following subset of the available actions $A(x_t)$,
    \begin{equation*}
    \hat{A}_t = \left\{ a \in A(x_t) : T^a_{x_t}(t) \geq \left( \ln T_{x_t}(t) \right)^2 \right\}.
    \end{equation*}
    The set $\hat{A}_t$ represents the actions available from state $x_t$ that have been sampled frequently enough that the estimates of the associated transition probabilities should be ``good''. In the limit, we expect that sub-optimal actions will be taken only logarithmically, and hence for sufficiently large $t$, $\hat{A}_t$ will contain only actions that are truly optimal. If no actions have been taken sufficiently many times, we take $\hat{A}_t = A(x_t)$ to prevent it from being empty.
    
    \item Value Estimates: having constructed these estimators, we compute $\hat{\phi}_t = \phi(\hat{A}_t, \hat{P}_t)$ and $\hat{\underline{v}}_t = \underline{v}(\hat{A}_t, \hat{P}_t)$ as the solution to the optimality equations in Eq. \eqref{eqn:dp}, essentially treating the estimated probabilities as correct and computing the optimal values and policy for the resulting estimated MDP.
\end{itemize}

At this point, we implement the following decision rule: for each action $a \in A(x_t)$, we compute the following \textit{index} over the set of possible transition laws:
\begin{equation} \label{eq:UCB-index}
u_a(t) = \sup_{\underline{q} \in \Theta} \left\{ L(x_t, a, \underline{q}, \hat{\underline{v}}) : \mathbf{I}(\underline{\hat{p}}^a_{x_t}, \underline{q}) \leq \frac{ \ln t}{T_{x_t,a}(t) } \right\},
\end{equation}
where $\mathbf{I}(\underline{p}, \underline{q}) = \sum_y p_y \ln( p_y / q_y )$ is the Kullback-Leibler divergence, and take action 
\begin{equation*}
\pi(t) = \argmax_{a \in A(x_t)} u_a(t).
\end{equation*}

This is a natural extension of several classical KL-divergence based UCB algorithms for the multi-armed bandit problem c.f. \citet{lai1985asymptotically}, \citet{burnetas1996optimal}, \citet{cowan2017normal} taking the view of the $L$ function as the `value' of taking a given action in a given state, estimated with the current data. In \citet{burnetas1996optimal}, a modified version of the above algorithm is in fact shown to be asymptotically optimal. The modification is largely for analytical benefit however, the pure index algorithm as above shows excellent performance c.f. Figure \ref{fig:simulation}. Further discussion of the performance of this algorithm is given in Section \ref{sec:comparison}.

An important and legitimate concern to the practical usage of the MDP-UCB algorithm that has been noted in \citet{tewari2008optimistic} among others, is actually calculating the index in Eq. \eqref{eq:UCB-index}. This and other issues are discussed in more depth in Section \ref{sec:computation}, where a computationally efficient formulation is presented. Additionally, in Section \ref{sec:comparison}, we highlight beneficial behavior of this algorithm that makes it worth pursuing.

\subsection{A Deterministic Minimum Empirical Divergence Type Algorithm for MDPs Under Uncertain Transitions}\label{sec:DMED}
In the classical DMED algorithm for multi-armed bandit problems \citep{honda2010asymptotically}, rather than considering (inflated) values for each action to determine which should be taken, DMED attempts to estimate how often each action ought to be taken. Recall the interpretation of \citet{burnetas1996optimal} given previously, that for any uniformly fast policy $\pi$, for any sub-optimal action $a \notin O(x,P)$ we have
\begin{equation*}
\liminf_T \frac{ \mathbb{E}\left[ T^a_{x}(T) \right] }{ \ln T} \geq \frac{1}{ \mathbf{K}_{x,a}(P) },
\end{equation*}
where $\mathbf{K}_{x,a}(P)$ measures (via the Kullback-Leibler divergence) how much the transition law for action $a$ would need to be changed to make action $a$ optimal. 

DMED proceeds by the following reasoning. If we estimate that the sub-optimal action $a$ is close to being optimal (low $K_{x,a}$), make sure we take it often enough to differentiate between them (ensure $T^a_x$ is high). If, on the other hand, we estimate that the sub-optimal action $a$ is far from being optimal (high $K_{x,a}$), we don't need to take is as often (ensure $T^a_x$ is low). As with the MDP-UCB and OLP algorithms, this requires an exploration of the possible transition laws ``near'' the current estimated transition law.

In general, computing the function $\mathbf{K}_{x,a}(P)$ is not easy. We consider the following substitute, then:
\begin{equation*}
\begin{split}
& \mathbf{\tilde{K}}_{x,a}(P, \underline{v}, a^*) =  \inf_{\underline{q} \in \Theta} \left\{ \mathbf{I}(\underline{p}^a_x, \underline{q}): L(x,a,\underline{q}, \underline{v}) \geq L(x,a^*,\underline{p}^{a^*}_x, \underline{v}) \right\}.
\end{split}
\end{equation*}
This is akin to exploratory \textit{policy iteration}. That is, determining, based on the current value estimates, how much modification would produce an improving action.

The function $\mathbf{K}$ measures how far the transition vector associated with $x$ and $a$ must be perturbed (under the KL-divergence) to make $a$ the optimal action for $x$. The function $\mathbf{\tilde{K}}$ measures how far the transition vector associated with $x$ and $a$ must be perturbed (under the KL-divergence) to make the value of $a$, as measured by the $L$-function, no less than the value of an optimal action $a^*$. As will be shown in Section \ref{sec:computation}, $\mathbf{\tilde{K}}$ may be computed fairly simply, in terms of the root of a single non-linear equation.

In this way, we have the following approximate MDP-DMED algorithm (see \citet{honda2010asymptotically} and \citet{honda2011asymptotically} for the multi-armed bandit version of this algorithm).

At any time $t \geq 1$, let $x_t$ be the current state, and construct the estimators as in the MDP-UCB algorithm in Section \ref{sec:UCB}, $\hat{P}_t$, $\hat{A}_t$, and utilize these to compute the estimated optimal values, $\hat{\phi}_t = \phi(\hat{A}_t, \hat{P}_t)$ and $\underline{\hat{v}}_t = \underline{v}( \hat{A}_t, \hat{P}_t)$.

Let $\hat{a}^*_t = \argmax_{a \in A(x_t)} L(x_t, a, \underline{\hat{p}}^{a}_{x_t}, \underline{\hat{v}}_t)$ be the estimated ``best'' action to take at time $t$. For each $a \neq \hat{a}^*_t$, compute the discrepancies $$D_t(a) = \ln t / \mathbf{\tilde{K}}_{x_t,a}(\hat{P}_t, \underline{\hat{v}}_t, \hat{a}^*_t) - T_{x_t,a}(t).$$

If $\max_{a \neq \hat{a}^*_t} D_t(a) \leq 0$, take $\pi(t) = \hat{a}^*_t$, otherwise, take $\pi(t) = \argmax_{a \neq \hat{a}^*_t} D_t(a).$

Following this algorithm, we perpetually reduce the discrepancy between the estimated sub-optimal actions, and the estimated rate at which those actions should be taken. The exchange from $\mathbf{K}$ to $\mathbf{\tilde{K}}$ sacrifices some performance in the pursuit of computational simplicity, however it also seems clear from computational experiments that MDP-DMED as above is not only computationally tractable, but also produces reasonable performance in terms of achieving small regret c.f. Figure \ref{fig:simulation}. Further discussion of the performance of this algorithm is given in Section \ref{sec:comparison}.

\subsection{Optimistic Linear Programming, Another UCB-Type Algorithm for MDPs Under Uncertain Transitions} \label{sec:OLP}
As we have previously noted,  \citet{tewari2008optimistic} raises some legitimate computational concerns. They propose an alternative, algorithm which they term ``optimistic linear programming'' (OLP), which is closely related to the MDP-UCB algorithm presented here. The main difference between OLP and MDP-UCB is that OLP does not use the KL divergence to determine the confidence interval. Instead, OLP uses $L_1$ distance, which allows the resulting index to be computed via solving linear programs. This reduces the computational complexity at the cost of performance. As we will show in Section \ref{sec:computation}, the MDP-UCB optimization problem can be simplified drastically, to render the use of OLP, at least with respect to the computational issues, unnecessary. The algorithm we present here is a version of OLP algorithm presented in \citet{tewari2008optimistic}.

At any time $t \geq 1$, let $x_t$ be the current state, and construct the estimators as in the MDP-UCB algorithm in Section \ref{sec:UCB}, $\hat{P}_t$, $\hat{A}_t$, and utilize these to compute the estimated optimal values, $\hat{\phi}_t = \phi(\hat{A}_t, \hat{P}_t)$ and $\underline{\hat{v}}_t = \underline{v}( \hat{A}_t, \hat{P}_t)$.

At this point, we implement the following decision rule: for each action $a \in A(x_t)$, we compute the following \textit{index}, again maximizing value within some distance of the current estimates:
\begin{equation*} 
u_a(t) = \sup_{\underline{q} \in \Theta} \left\{ L(x_t, a, \underline{q}, \hat{\underline{v}}) : \lvert \lvert \underline{\hat{p}}^a_{x_t} - \underline{q} \rvert\rvert_1 \leq \sqrt{\frac{ 2 \ln t}{T^a_{x_t}(t)}} \right\},
\end{equation*}
and take action 
\begin{equation*}
\pi(t) = \argmax_{a \in A(x_t)} u_a(t).
\end{equation*}

\subsection{A Thompson-Type Algorithm for MDPs Under Uncertain Transitions}
In MDP-UCB, MDP-DMED, and OLP, above, we realized the notion of ``exploration'' in terms of considering alternative hypotheses that were ``close'' to the current estimates within $\Theta$, interpreting closeness in terms of ``plausibility''. In this section, we consider an alternative form of exploration through random sampling over $\Theta$, based on the current available data. Given a uniform prior over $\Theta$, the posterior for $\underline{p}^a_x$ is given by a Dirichlet distribution with the observed occurrences. Posterior Sampling (MDP-PS) proceeds in the following way:

At any time $t \geq 1$, let $x_t$ be the current state, and construct the estimators as in the MDP-UCB algorithm in Section \ref{sec:UCB}, $\hat{P}_t$, $\hat{A}_t$, and utilize these to compute the estimated optimal values, $\hat{\phi}_t = \phi(\hat{A}_t, \hat{P}_t)$ and $\underline{\hat{v}}_t = \underline{v}( \hat{A}_t, \hat{P}_t)$. In addition, generate the following random vectors:

For each action $a \in A(x_t)$, let $\underline{T}^a_{x_t}(t) = [ T^a_{x_t,y}(t) ]_{y \in S}$ be the vector of observed transition counts from state $x_t$ to $y$ under action $a$. Generate the random vector $\underline{Q}$ according to
\begin{equation*}
\underline{Q}^{a}(t) \sim \text{Dir}( \underline{T}^a_{x_t}(t) ).
\end{equation*}
The $\underline{Q}^a(t)$ are distributed according to the joint posterior distribution of $\underline{p}^a_{x_t}$ with a uniform prior.

At this point, define the following values as posterior sampled estimates of the potential value $L$ of each action:
\begin{equation*}
W_a(t) = r_{x_t,a} + \sum_{y} Q^a_y(t) \hat{v}_y,
\end{equation*}
and take action $\pi(t) = \argmax_{a \in A(x_t)} W_a(t).$

In this way, we probabilistically explore likely hypotheses within $\Theta$, and act according to the action with best hypothesized value.

\section{Accelerating Computation} \label{sec:computation}
All of the above algorithms require computing the estimated optimality values $\hat{\phi}_t, \underline{\hat{v}}_t$ each round. This is an issue, but efficient linear programming formulations exist to solve the optimality equations in Eq. \eqref{eqn:dp} see for example \citet{derman1970finite}. It may also be possible to adapt the method of \cite{lakshminarayanan2017linearly} for approximately solving MDPs, among others, to our undiscounted and potentially changing MDP setting.

However, each of these algorithms additionally has unique computational challenges, through computations over the high dimensional parameter space $\Theta$ due to the typically high cardinality of the state space. 

\subsection{MDP-UCB}
We will first examine the MDP-UCB algorithm from Section \ref{sec:UCB}. Recalling the notation that $\mathbf{I}(\underline{p}, \underline{q}) = \sum_x p_x \ln(p_x/q_x)$, MDP-UCB has to repeatedly solve the following optimization problem:
\begin{equation*}
\begin{split}
C(\underline{p}, \underline{v}, \delta) & = \sup_{\underline{q} \in \Theta} \left\{ \sum_{x} q_x v_x :  \mathbf{I}( \underline{p}, \underline{q})  \leq \delta \right\}
\end{split}.
\end{equation*}
The index of the MDP-UCB algorithm may be efficiently expressed in terms of the $C$ function above which we will refer to as the $\underline{q}$-Formulation.

This represents an $\lvert S \rvert$-dimensional non-linear constrained optimization problem which is not, in general, easy to solve.

For mathematical completeness, as well as for practical implementation, we first analyze some trivial cases. Let $\mu_p = \sum_x p_x v_x$ and $V = \max_x v_x$, then
\begin{theorem} \label{theorem-UCB-Simplification-degenerate}
    The value of $C(\underline{p}, \underline{v}, \delta)$ can be easily found in the following cases:
    \begin{itemize}
        \item If $\delta <0$ then the optimization problem, $C(\underline{p}, \underline{v}, \delta)$ is infeasible and we say $C(\underline{p}, \underline{v}, \delta) = -\infty$.
        \item If $\delta = 0$, then $C(\underline{p}, \underline{v}, \delta) = \mu_p$.
        \item If $\delta > 0$ and $v_{x_1} = v_{x_2}$ for all $x_1,x_2 \in S$, then $C(\underline{p}, \underline{v}, \delta) = \mu_p$.
    \end{itemize} 
\end{theorem}

Proof of this theorem is provided in Appendix \ref{sec:UCB-degenerate-Proof}.

For other cases, we can reduce this to solving a $2$ dimensional system of non-linear equations, with unknowns $\mu_q^*$ and $\lambda$ as follows.
\begin{theorem} \label{theorem-UCB-Simplification}
    For any $\delta >0$ and $\underline{v}$ such that $v_{x_1} \neq v_{x_2}$ for some  $x_1,x_2 \in S$,
    $$C(\underline{p}, \underline{v}, \delta) = \mu_q^*,$$ 
    where
    \begin{align*}
    \sum_{x \in S} p_x \ln \left( 1 + \dfrac{v_x - \mu_q^*}{\lambda} \right) = \delta, \\
    \sum_x p_x \dfrac{\lambda}{\lambda + v_x - \mu_q^*} =1, \\
    \mu_p < \mu_q^* < V \text{ and }  \lambda < \mu_q^* -V .
    \end{align*}
\end{theorem}
Proof of this theorem is provided in Appendix \ref{sec:UCB Simplification Proof}.

Solving these systems, which we will refer to as the $(\mu_q^*, \lambda)$-Formulation, provides dramatic speed increases for the implementation of the algorithm (Figure \ref{fig:computation-time}). We also note that the $(\mu_q^*, \lambda)$-Formulation scales manageably with the dimension of the state space, as opposed to the $q$-Formulation. Additionally, the structure of the equations admits several nice solution methods since, for a given $\mu_q$, the second equation has a unique solution for $\lambda$ in the indicated range, and given that solution, the summation in the first equation is increasing to infinity as a function of $\mu_q$. 

\subsection{MDP-DMED}
Next we examine the MDP-DMED algorithm from Section \ref{sec:DMED}. Again, recalling the notation that $\mathbf{I}(\underline{p}, \underline{q}) = \sum_x p_x \ln(p_x/q_x)$, MDP-DMED has to repeatedly solve the following optimization problems:

\begin{equation*}
\begin{split}
D(\underline{p}, \underline{v}, \rho) & = \inf_{\underline{q} \in \Theta} \left\{ \mathbf{I}(\underline{p}, \underline{q}): \sum_x q_x v_x \geq \rho \right\}.
\end{split}
\end{equation*}

The rate function $\mathbf{\tilde{K}}$ of the MDP-DMED algorithm may be efficiently expressed in terms of the $D$ function above which we will refer to as the $\underline{q}$-Formulation. This represents an $\lvert S \rvert$-dimensional non-linear constrained optimization problems, which is not, in general, easy to solve.

As before, we consider some trivial cases first. Let $\mu_p = \sum_x p_x v_x$ and $V = \max_x v_x$, then
\begin{theorem} \label{theorem-DMED-Simplification-degenerate}
    The value of $D(\underline{p}, \underline{v}, \rho)$ and by extension $D_t(a)$ can be easily found in the following cases:
    \begin{itemize}
        \item If $\rho > V$ then the optimization problem, $D(\underline{p}, \underline{v}, \rho)$ is infeasible and we say $D(\underline{p}, \underline{v}, \rho) = \infty$ and $D_t(a) = -T_{x_t,a}(t)$.
        
        \item If $\rho \leq \mu_p$ then $D(\underline{p}, \underline{v}, \rho) = 0$ and we say $D_t(a) = \infty$.
        
        \item If $v_{x_1} \neq v_{x_2}$ for some  $x_1,x_2 \in S$ and $\rho = V$, then optimization problem $D(\underline{p}, \underline{v}, \rho)$ diverges to infinity and we say $D(\underline{p}, \underline{v}, \rho) = \infty$ and $D_t(a) = -T_{x_t,a}(t)$.
    \end{itemize} 
\end{theorem}

Proof of this theorem is provided in Appendix \ref{sec:DMED-degenerate-Proof}.

For other cases, this optimization problem reduces to solving a $1$-dimensional system of non-linear equations with one unknown, $\lambda$,  as follows:
\begin{theorem} \label{theorem-DMED-Simplification}
    For any $\underline{v}$ such that $v_{x_1} \neq v_{x_2}$ for some $x_1,x_2 \in S$ and $\mu_p < \rho < V$,
    \begin{equation*}
    D(\underline{p}, \underline{v}, \rho) = \sum_x p_x \ln( 1 + (\rho - v_x)\lambda),
    \end{equation*}
    where
    \begin{align*}
    \sum_x p_x \frac{ \rho - v_x }{1 + (\rho - v_x)\lambda} = 0, \\
    0 < \lambda < \frac{1}{V - \rho}.
    \end{align*}
\end{theorem}
Proof of this theorem is provided in Appendix \ref{sec:DMED Simplification Proof}.

As with the MDP-UCB case, solving this system, which we will refer to as the $ \lambda$-Formulation, provides dramatic speed increases for the implementation of the algorithm (Figure \ref{fig:computation-time}). We also note that the $\lambda$-Formulation scales manageably with the dimension of the state space, as opposed to the $q$-Formulation. Additionally, the $\lambda$-Formulation structurally lends itself well to solutions. Over the indicated range, the summation is positive and constant in the limit as $\lambda \to 0$, and monotonically decreasing, diverging to negative infinity as $\lambda \to 1/(V - \rho)$. Hence the solution is unique, and can easily be found via bisection.

\subsection{OLP}
Next we examine the OLP algorithm from Section \ref{sec:OLP}. OLP has to repeatedly solve the following optimization problem:
\begin{equation*}
\begin{split}
B(\underline{p}, \underline{v}, \delta) & = \sup_{\underline{q} \in \Theta} \left\{ \sum_{x} q_x v_x : \lvert \lvert \underline{\hat{p}}^{x_t} - \underline{q} \rvert\rvert_1 \leq \delta \right\}
\end{split}.
\end{equation*}
The index of the OLP algorithm may be efficiently expressed in terms of the $B$ function above. $B(\underline{p}, \underline{v}, \delta)$ is equivalent to the following linear program:

\begin{alignat*}{2}
\text{max}_{\underline{q}^+,\underline{q}^-} \quad   & \sum_{x \in S} v_x(q{^-}{_x} - q{^+}{_x} +  p{_x}),  & \\ 
\nonumber \\
\nonumber \text{s.t. } & \\
& \sum_{x \in S} q{^+}{_x} +  q{^-}{_x} \leq \delta, \\
& \sum_{x \in S} q{^-}{_x} - q{^+}{_x} =  0, \\
& q{^+}{_x} - q{^-}{_x}  \leq p_x  & \forall x \in S, \\
& q{^+}{_x}, q{^-}{_x}  \geq 0 & \forall  x \in S.
\end{alignat*}

This represents an $\lvert S \rvert$-dimensional linear program, which can generally be computed quite efficiently. However, as the dimension of the state space increases we incur a greater computational burden (Figure \ref{fig:computation-time}).

\subsection{MDP-PS}
The most attractive advantage of MDP-PS is the reduced computational cost, relative to the other three proposed algorithms (Figure \ref{fig:computation-time-best}). Notice there is no extra optimization problem that needs to be solved. In the MDP-UCB algorithm, at every time $t$, we had to iteratively solve $\lvert A(x_t) \rvert$ instances of $C(\underline{p}, \underline{v}, \delta)$, for OLP $\lvert A(x_t) \rvert$ instances of $B(\underline{p}, \underline{v}, \delta)$, and for MDP-DMED, $\lvert A(x_t) \rvert$ instances of $D(\underline{p}, \underline{v}, \rho)$. Under MDP-PS, the computational burden stems from sampling from the Dirichlet distribution for each action (again, $\vert A(x_t)\vert$ steps), but this is a well studied problem with many efficiently implemented solutions (see for example \citet{mckay2003information}). Specific properties of the MDP-PS algorithm may still make these other algorithms worth pursuing, however, as seen in Section \ref{sec:comparison}.

\subsection{Computation Time Comparison}
To demonstrate the computational time savings achieved by these simplifications we randomly generated the parameters for 15 different action indices and timed how long each algorithm took to solve. We repeated this for 4 different values of $\vert S \rvert$, the dimension of the state space, $10$, $100$, $1,000$, and $10,000$. In Figure \ref{fig:computation-time}, we plot the mean computation time as $\vert S \rvert$ increases, for each algorithm, [1] MDP-PS, [2] MDP-DMED $\lambda$-Formulation, [3] MDP-UCB $(\mu_q^*, \lambda)$-Formulation , [4] MDP-DMED  $\underline{q}$-Formulation , [5] MDP-UCB  $\underline{q}$-Formulation, and [6] OLP, along with a $95\%$ confidence interval. 

In order to keep the comparisons as equitable as possible, the optimization problem for all the algorithms (with the exception of MDP-PS) were solved to within 4 digits of accuracy using TensorFlow for Python \citep{abadi2016tensorflow}. MDP-PS used SciPy's random Dirichlet generator. They were all run on a MacBook Pro with a 3.1 Ghz i7 processor with 16GB DDR3 RAM.

\begin{figure}
    \begin{center}
        \includegraphics[width=\textwidth]{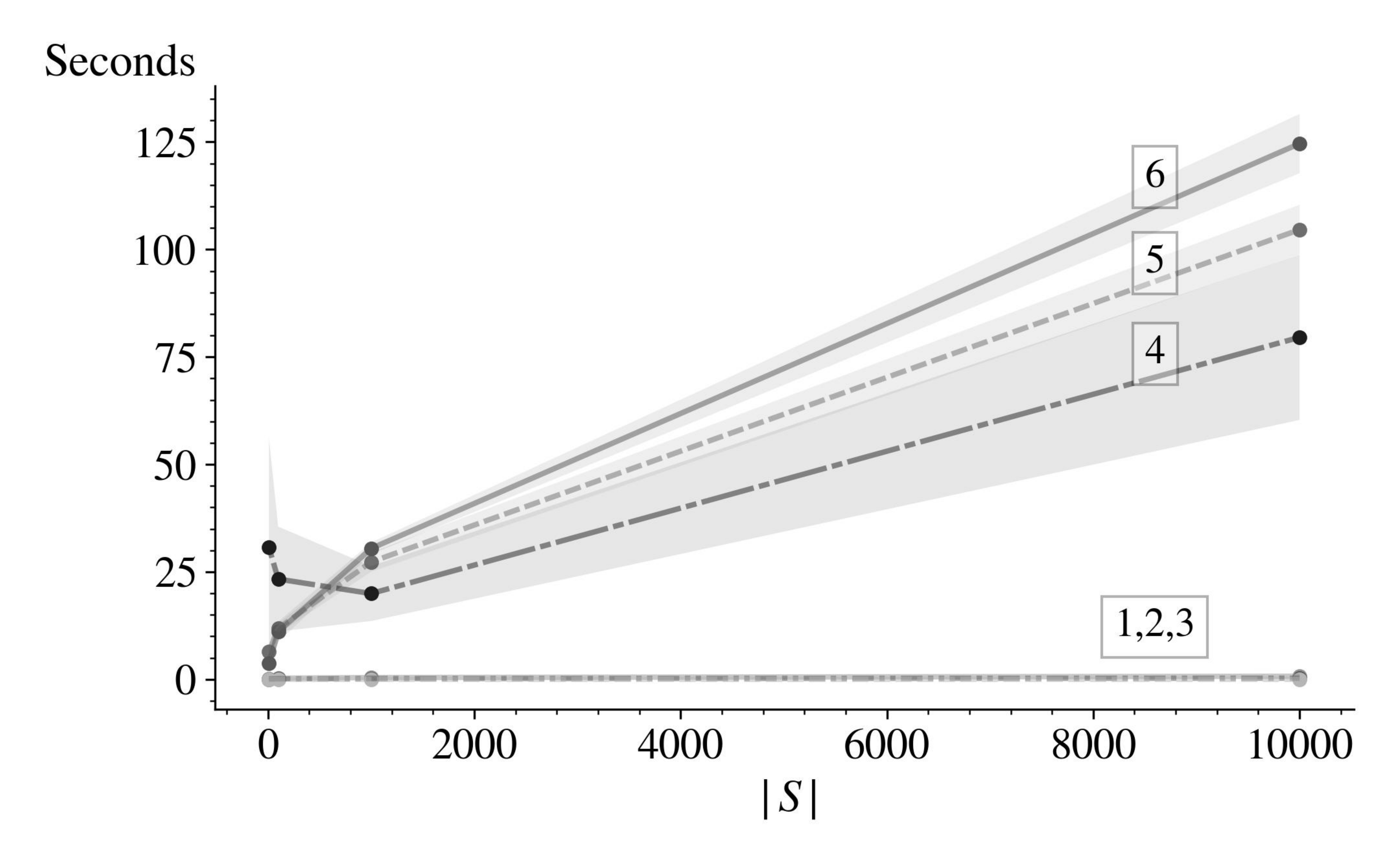}	
        \caption{\small Computation time as $\lvert S \rvert$ increases \label{fig:computation-time}} 
    \end{center}
\end{figure}

The top three fastest algorithms were [1] MDP-PS, [2] MDP-DMED $\lambda$-Formulation, and [3] MDP-UCB $(\mu_q^*, \lambda)$-Formulation. Figure \ref{fig:computation-time-best} shows these three in more detail.

\begin{figure}
    \begin{center}
        \includegraphics[width=\textwidth]{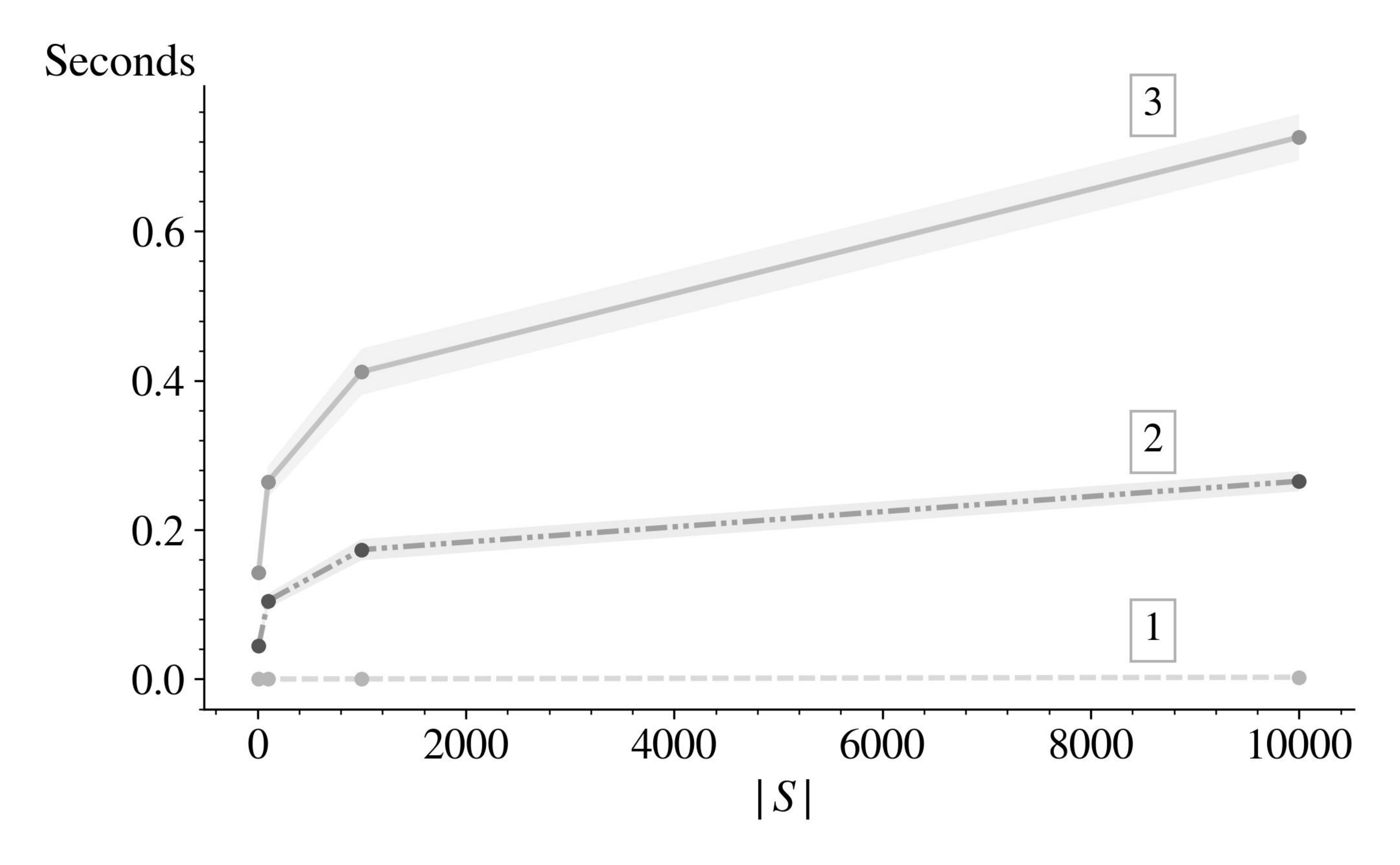}	
        \caption{\small Computation time as $\lvert S \rvert$ increases for the top three performers  \label{fig:computation-time-best}} 
    \end{center}
\end{figure}

From Figure \ref{fig:computation-time} we can see the dramatic savings achieved by [2] MDP-DMED using the $\lambda$-Formulation, and [3] MDP-UCB using the $(\mu_q^*, \lambda)$-Formulation as compared to [4,5] the $\underline{q}$-Formulations. [6] OLP also suffers from increasing computation time as the dimension of the state space increases. OLP performs the worst in terms of computational time which is likely due to the fact that we are not using a specialized fast LP solver but rather TensorFlow.

In Figure \ref{fig:computation-time-best} we can see the relative performances of the top three algorithms. [1] MDP-PS, unsurprisingly with the fastest, followed by [2] MDP-DMED using the $\lambda$-Formulation with its single unknown, and then [3] MDP-UCB using the $(\mu_q^*, \lambda)$-Formulation with its two unknowns.

The absolute time is not as important as the relative time. There are numerous ways to achieve significantly faster absolute time but our focus here is to demonstrate the relative speed increase gained by using our simplifications. In addition, one can get raw computational time savings by developing a devoted optimizer for problems of this type but if we restrict to using a generic black box optimizer, the method we employed seems a reasonable reflection of what one would do.

\section{Comparison of Performance} \label{sec:comparison}
In this section we discuss the results of our simulation test of these algorithms on a small example problem. There is nothing particularly special about the values for this example, and we observe similar results under other values. Our example had 3 states ($x_1,x_2,$ and $x_3$) with 2 available actions ($a_1$ and $a_2$) in each state. Below we show the transition probabilities, as well as the reward, returned under each action.
{\small
    \begin{center}
        
        $P[a_1] =$ \begin{tabular}{| l | c | c | c | }
            \hline
            & $x_1$ & $x_2$ & $x_3$ \\ \hline
            $x_1$ & 0.04 & 0.69 & 0.27 \\ \hline
            $x_2$ & 0.88 & 0.01 & 0.11 \\ \hline
            $x_3$ & 0.02 & 0.46 & 0.52 \\ \hline
        \end{tabular},
    \end{center}
    \begin{center}
        $P[a_2] =$ \begin{tabular}{| l | c | c | c | }
            \hline
            & $x_1$ & $x_2$ & $x_3$ \\ \hline
            $x_1$ & 0.28 & 0.68 & 0.04 \\ \hline
            $x_2$ & 0.26 & 0.33 & 0.41 \\ \hline
            $x_3$ & 0.43 & 0.35 & 0.22 \\ \hline
        \end{tabular},
    \end{center}
    
    \begin{center}
        $R=$ \begin{tabular}{| l | c | c | c | }
            \hline
            & $x_1$ & $x_2$ & $x_3$ \\ \hline
            $a_1$ & 0.13 & 0.47 & 0.89 \\ \hline
            $a_2$ & 0.18 & 0.71 & 0.63 \\ \hline
        \end{tabular}.
    \end{center}
}
If these transition probabilities were known, the optimal policy for this MDP would be  $\pi^*(x_1) = a_1, \pi^*(x_2) = a_2,$ and $ \pi^*(x_3) =  a_1$.

We simulated each algorithm 100 times over a time horizon of 10,000 and for each time step we computed the mean regret as well as the variance. In Figure \ref{fig:simulation}, we plot the mean regret over time for each algorithm, [1] MDP-PS, [2] MDP-UCB, [3] OLP, and [4] MDP-DMED, along with a $95\%$ confidence interval for all sample paths.

\begin{figure}
    \begin{center}
        \includegraphics[width=\textwidth]{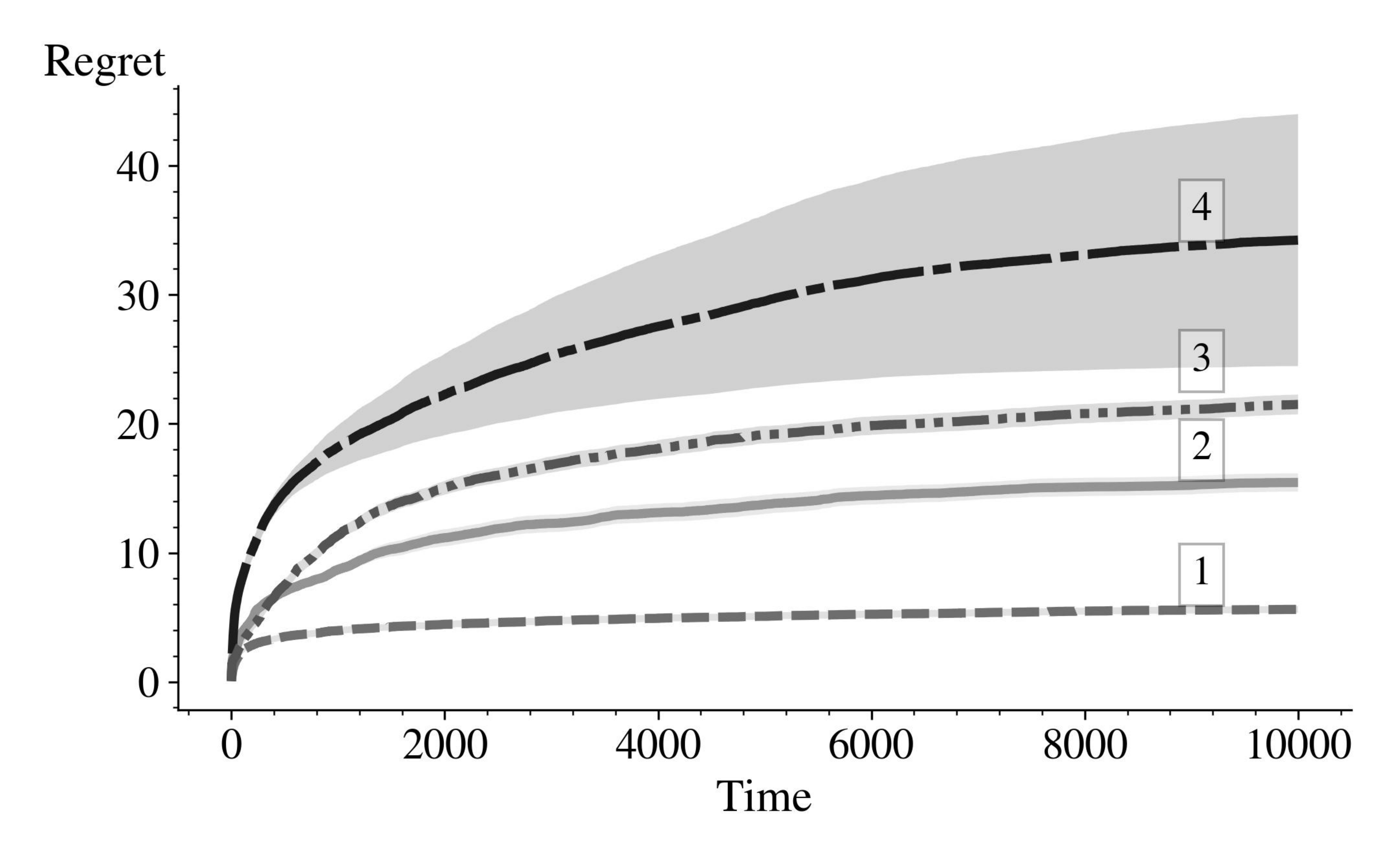}	
        \caption{\small Average cumulative regret over time for each algorithm \label{fig:simulation}} 
    \end{center}
\end{figure}

We can see that all algorithms seem to have logarithmic growth of regret. There are a few interesting differences that the plot highlights, at least for these specific parameter values:

MDP-DMED has not only the highest finite time regret, but also large variance that seems to increase over time. This seems primarily due to the ``epoch'' based nature of the algorithm, which results in exponentially long periods when the algorithm may get trapped taking sub-optimal actions, incurring large regret until the true optimal actions are discovered. The benefit of this epoch structure is that once the optimal actions are discovered, they are taken for exponentially long periods, to the exclusion of sub-optimal actions.

As expected, see \citet{tewari2008optimistic}, OLP has a higher finite time regret when compared to MDP-UCB, but still achieves logarithmic growth.

MDP-PS seems to perform best, exhibiting lowest finite time regret as well as the tightest variance. This seems largely in agreement with the performance of PS-type algorithms in other bandit problems as well, in which they are frequently asymptotically optimal c.f. \citet{cowan2017normal} and references therein.

\subsection{Algorithm Robustness---Inaccurate Priors}

How do these algorithms respond to potentially ``unlucky'' or non-representative streaks of data? How does bad initial estimates effect their performance? Can these algorithms be fooled, and what are the resulting costs before they recover? This is a practically important question, in terms of data security and risk assessment, but also an important element of evaluating a learning algorithm. How does the learning agent respond to non-ideal conditions?

To test these algorithms, we ``rigged'' or biased the first 60 actions and transitions, such that under the estimated transition probabilities the optimal policy would be to activate the sub-optimal action in each state. In more detail, let $T^a_{x,y}$ be the number of times we transitioned from state $x$ to state $y$ under action $a$. Then we rigged $T^a$ so that it started like so,

\begin{center}
    
    $T[a_1] =$ \begin{tabular}{| l | c | c | c | }
        \hline
        & $x_1$ & $x_2$ & $x_3$ \\ \hline
        $x_1$ & 8 & 1 & 1 \\ \hline
        $x_2$ & 1 & 1 & 8 \\ \hline
        $x_3$ & 8 & 1 & 1 \\ \hline
    \end{tabular} \ ,\\
\end{center}
\begin{center}
    $T[a_2] =$ \begin{tabular}{| l | c | c | c | }
        \hline
        & $x_1$ & $x_2$ & $x_3$ \\ \hline
        $x_1$ & 1 & 1 & 8 \\ \hline
        $x_2$ & 8 & 1 & 1 \\ \hline
        $x_3$ & 1 & 1 & 8 \\ \hline
    \end{tabular}
\end{center}

Under the resulting (bad) estimated transition probabilities, we have that the (estimated) optimal policy is $\hat{\pi}^*(x_1) = a_2,  \hat{\pi}^*(x_2) = a_1$, and $\hat{\pi}^*(x_3) = a_1$, which in fact chooses the sub-optimal action in each state.

The subsequent performances of the MDP algorithms are plotted in Figure \ref{fig:robustness}. All algorithms still appear  to have logarithmic growth in regret, suggesting they can all `recover' from the initial bad estimates. It is striking though, the extent to which the average regrets for MDP-DMED and MDP-PS are affected, increasing dramatically as a result, MDP-PS demonstrating an increase in variance as well. However, the MDP-UCB algorithm seems relatively stable: its average regret has barely increased, and maintains a small variance. Empirically, this phenomenon appears common for the MDP-UCB algorithm under other extreme conditions. The underlying cause and a rigorous examination of these intuitions, will be explored in a future work.

\begin{figure}
    \begin{center}
        \includegraphics[width=\textwidth]{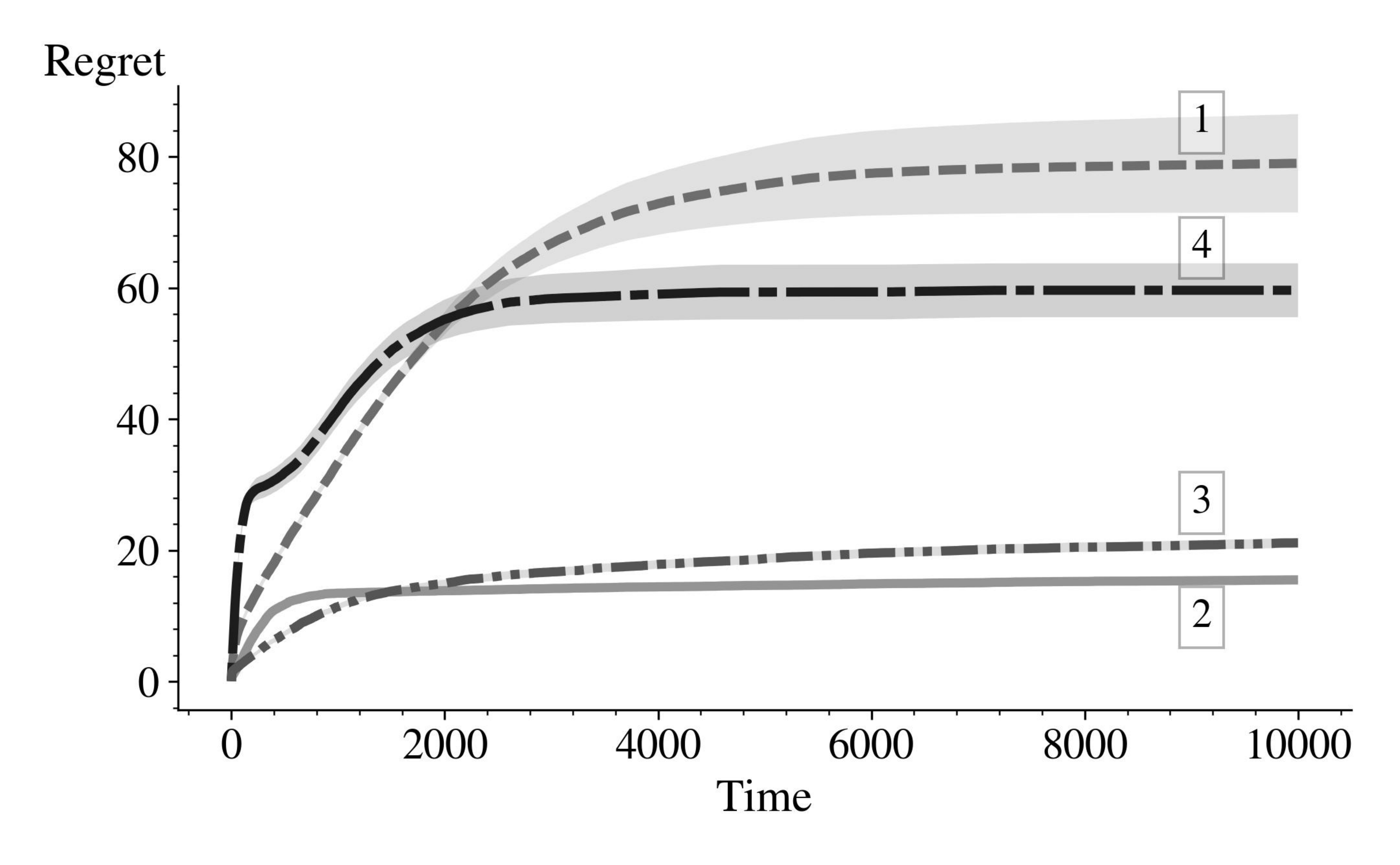}	
        \caption{\small Robustness test. MDP-UCB seems to be largely unaffected by the inaccurate priors.  \label{fig:robustness}} 
    \end{center}
\end{figure}

\section{Conclusion and Future Work}
In this paper we have presented four algorithms adapted from classical multi-armed bandit algorithms that either are provably asymptotically optimal or at least give that appearance in practice. The simplifications for MDP-UCB and MDP-DMED presented here have been shown to dramatically reduce the computational burden for these algorithms, rendering them more useful in practice. As a result, the provably worse performing OLP, no longer has any advantage over them. MDP-DMED under the $\lambda$-Formulation is fast and possibly optimal, but has a high variance for regret that increases over time. While MDP-PS is very fast and appears to be optimal, it is highly sensitive to incorrect priors or extreme sampling errors. MDP-UCB is provably optimal has stable performance under various extreme conditions, and can be computed quickly using the $(\mu_q^*, \lambda)$-Formulation. 

There are various interesting directions to continue this work, we mention a few potential avenues here. The idea of ``exploring the hypothesis space'' is something that extends immediately to the case of unknown rewards. Each of the algorithms presented here can generalize immediately to such situations, though the computational simplifications would need to be modified significantly.

It would also be of theoretical interest to find sufficient conditions on the estimators used to ensure asymptotically optimal performance. This could potentially allow these algorithms to be modified to use other state value estimators (for example, Q-learning \citet{watkins1989learning}) while maintaining their theoretical guarantees. From a practical computational point of view we could consider systems where we can't easily iterate over all possible states, and how these algorithms can be modified to address this. These ideas will be explored in future works.

\acks{We acknowledge support for this work from the National Science Foundation, NSF grant CMMI-1662629.}

\appendix

\section{Proof of Theorems of Section \ref{sec:computation}}

\subsection{Proof of Theorem \ref{theorem-UCB-Simplification-degenerate}}\label{sec:UCB-degenerate-Proof}
First we restate Theorem \ref{theorem-UCB-Simplification-degenerate}:

The value of $C(\underline{p}, \underline{v}, \delta)$ can be easily found in the following cases:
\begin{itemize}
    \item If $\delta <0$ then the optimization problem, $C(\underline{p}, \underline{v}, \delta)$ is infeasible and we say $C(\underline{p}, \underline{v}, \delta) = -\infty$.
    \item If $\delta = 0$, then $C(\underline{p}, \underline{v}, \delta) = \mu_p$.
    \item If $\delta > 0$ and $v_{x_1} = v_{x_2}$ for all $x_1,x_2 \in S$, then $C(\underline{p}, \underline{v}, \delta) = \mu_p$.
\end{itemize}

\begin{proof}
    Recall that $\mathbf{I}( \underline{p}, \underline{q})$ is the KL Divergence from $\underline{p}$ to $\underline{q}$. We then have by Gibb's inequality that $\mathbf{I}( \underline{p}, \underline{q}) \geq 0$, with equality if and only if $\underline{p} =  \underline{q}$. Thus, if $\delta <0$ then the optimization problem is infeasible. If $\delta = 0$ then it has the trivial solution $\underline{q}^* =  \underline{p}$. We therefore take $\delta > 0$. Now, if $v_{x_1} = v_{x_2}$ for all $x_1,x_2 \in S$ then any feasible probability vector $\underline{q}$ is also optimal with $C(\underline{p}, \underline{v}, \delta) = v_x = \mu_p$.
\end{proof}

\subsection{Proof of Theorem \ref{theorem-UCB-Simplification}} \label{sec:UCB Simplification Proof}

In this section we will prove Theorem \ref{theorem-UCB-Simplification}, which we restate here.

Let $\mu_p = \sum_x p_x v_x$ and $V = \max_x v_x$. Then for any $\underline{v}$ such that $v_{x_1} \neq v_{x_2}$ for some  $x_1,x_2 \in S$ and $\delta >0$,
$$C(\underline{p}, \underline{v}, \delta) = \mu_q^*,$$
where
\begin{align*}
\sum_{x \in S} p_x \ln \left( 1 + \dfrac{v_x - \mu_q^*}{\lambda} \right) = \delta, \\
\sum_x p_x \dfrac{\lambda}{\lambda + v_x - \mu_q^*} =1, \\
\mu_p < \mu_q^* < V \text{ and }  \lambda < \mu_q^* -V .
\end{align*}

Before giving the formal proof, it may be helpful to understand the overall conception of the proof. The main idea is the use of Lagrange multiplier techniques, which greatly reduces the dimensionality of the problem to be solved. We are able to exchange from trying to find the optimal probability vector $\underline{q}^*$, to a problem where we need only find two moments of the optimal $\underline{q}^*$, a dramatic dimension reduction. In the MDP-UCB case, it suffices to find the unknown optimal mean of the optimal distribution,  $\underline{q}^*$, $\mu_q^*$, and a value $\lambda = \sigma_{q^*}^2/(\mu_p - \mu_q^*)$ which depends on the optimal, unknown variance.

\begin{proof}
    Recall that,
    \begin{equation}\label{eqn:UCB-originalOPT}
    \begin{split}
    C(\underline{p}, \underline{v}, \delta) & = \sup_{\underline{q} \in \Theta} \left\{ \sum_{x} q_x v_x :  \mathbf{I}( \underline{p}, \underline{q})  \leq \delta \right\} \\
    \end{split}
    \end{equation}
    Since $\{ \underline{q} : \underline{q} \in \Theta,  \mathbf{I}( \underline{p}, \underline{q})  \leq \delta \}$ is a closed compact set, the supremum will be realized by a maximum, and we may express the problem of computing $C(\underline{p}, \underline{v}, \delta)$ in the following form:
    \begin{alignat}{2}
    \text{max}_{\underline{q}} \quad   & \mu_q = \sum_{x \in S} q_x v_x,   & \label{eqn:UCB-align-original-opt-objective}\\
    \nonumber \\
    \nonumber \text{s.t. } & \\
    & \sum_{x \in S} p_x \ln \left( \dfrac{p_x}{q_x} \right) \leq \delta \label{eqn:UCB-align-original-opt-delta-constraint}, \\
    \nonumber & \sum_{x \in S} q_x = 1,\\
    \nonumber &  q_x > 0 & x \in S.
    \end{alignat}
    
    Let $\mu_q^* = \sum_{x \in S} q^*_x v_x $ be the optimal value of the objective function, $\mu_p = \sum_{x \in S} p_x v_x$, and $V = \max_x v_x$. First we will argue that,
    \begin{equation*}
    \mu_p  \leq \mu_q^* < V.
    \end{equation*}
    
    To see the first inequality, observe that $\underline{q} = \underline{p}$ satisfies the constraints and is therefore feasible, hence the objective function at $\underline{q} = \underline{p}$ is less than or equal to the optimum: $\mu_p \leq \mu^*_q$. To see the second, note that $\mu^*_q$ will be an expected value over the $\{ v_x \}$, and hence less than or equal to the maximum, $V$. Because the probabilities in $\underline{q}^*$ are strictly positive, the expected value $\mu^*_q$ must actually be strictly less than the maximum: $\mu^*_q < V$.

    Utilizing Lemma \ref{lem:UCB-kl-constraint-equality} in Appendix \ref{sec:KL-Lemmas}, for any feasible $\underline{q}$ such that the KL Divergence constraint is not achieved with equality, a different feasible $\underline{q^\prime}$ exists with an improved value of the objective function. Hence we can rewrite the optimization problem as,
    
    \begin{alignat}{2}
    \nonumber \text{max}_{\underline{q}} \quad   & \mu_q = \sum_{x \in S} q_x v_x,   &\\
    \nonumber \\
    \nonumber \text{s.t. } & \\
    &\sum_{x \in S} p_x \ln \left( \dfrac{p_x}{q_x} \right) = \delta  \label{eqn:UCB-align-opt-delta-constraint},\\
    & \sum_{x \in S} q_x = 1  \label{eqn:UCB-align-opt-sum-to-one-constraint}, \\
    & q_x > 0 & x \in S \label{eqn:UCB-align-opt-greater-than-zero-constraint}.
    \end{alignat}
    
    We now turn to the main task, reducing the dimension of the optimization problem. Using Lagrange multipliers we have the following auxiliary function,
    \begin{equation*}
    \begin{split}
    L(\underline{q},\lambda,\mu) & = \sum_{x \in S} q_x v_x + \lambda \left( \sum_{x \in S} p_x \ln \left( \dfrac{p_x}{q_x} \right) - \delta \right) + \mu \left( \sum_{x \in S} q_x  - 1 \right).
    \end{split}
    \end{equation*}
    
    Note that when using the Lagrange multipliers, we can safely ignore the positivity inequality constraints in Eq. \eqref{eqn:UCB-align-opt-greater-than-zero-constraint} because they are strict inequalities, thus inactive, and removing them will not change the local optimum.
    
    Taking partial derivatives, we get,
    \begin{equation*}
    \begin{split}
    L^{\prime}_{q_x}(\underline{q},\lambda,\mu) & =  v_x - \dfrac{\lambda  p_x}{q_x} + \mu \; \; \; ,\; \; \forall x \in S, \\
    L^{\prime}_{\lambda}(\underline{q},\lambda,\mu) & =   \sum_{x \in S} p_x \ln \left( \dfrac{p_x}{q_x} \right) - \delta, \\
    L^{\prime}_{\mu}(\underline{q},\lambda,\mu) & =  \sum_{x \in S} q_x  - 1.
    \end{split}
    \end{equation*}
    
    Setting them to zero, results in the following system of equations for the optimal solution, $\underline{q}^*$,
    \begin{align}
    v_x + \mu &= \dfrac{\lambda  p_x}{q^*_x} \; \; \; ,\; \; \forall x \in S \label{eqn:zero-partial-a}, \\
    \nonumber \sum_{x \in S} p_x \ln \left( \dfrac{p_x}{q^*_x} \right) &= \delta, \\
    \nonumber \sum_{x \in S} q^*_x &= 1.
    \end{align}
    
    We are looking for a solution $\underline{q}^*$ to this system, and any such solution will be a global maximum. To see this, observe that our optimization problem is a convex optimization problem. This can be seen more easily when put in its original form, as in Eq. \eqref{eqn:UCB-originalOPT}. We are maximizing a linear (and thus concave) function, the inequality constraint is convex, and the equality constraints are affine. Thus, any stationary point will be a local maximum and any local maximum will be a global maximum. \citep{boyd2004convex}
    
    Multiplying Eq. \eqref{eqn:zero-partial-a} through by $q^*_x$, we have,
    \begin{equation} \label{eqn:UCB-multiply-through}
    \lambda p_x = q^*_x(v_x + \mu)  \; \; \; ,\; \; \forall x \in S.
    \end{equation}
    
    Summing Eq. \eqref{eqn:UCB-multiply-through} over $x$, we have
    
    \begin{equation} \label{eqn:UCB-lambda-mu-system-A}
    \lambda = \mu_q^* + \mu.
    \end{equation}
    
    We now introduce a quantity, $\sigma^2_{q^*}$, the variance under transition law $\underline{q}^*$, explicitly defined as follows
    
    \begin{equation}\label{eqn:UCB-sigma-q-definition}
    \sigma^2_{q^*} = \sum_{x \in S} q^*_x v_x^2 - \mu_{q^*}^2.
    \end{equation}
    
    Looking at Eq. \eqref{eqn:UCB-multiply-through} again, but this time, multiplying through by $v_x$ we get,
    
    \begin{equation*}
    \lambda p_x v_x = q^*_x v_x^2 + q^*_x v_x \mu  \; \; \; ,\; \; \forall x \in S.
    \end{equation*}
    
    Summing this over $x$ yields,
    \begin{equation} \label{eqn:UCB-lambda-mu-system-B}
    \mu_p \lambda = \sigma^2_{q^*} + \mu_{q^*}^2 + \mu \mu_{q^*}.
    \end{equation}
    
    Equations \eqref{eqn:UCB-lambda-mu-system-A} and \eqref{eqn:UCB-lambda-mu-system-B} form a system of equations with two unknowns $\mu$ and $\lambda$. Solving this system yields,
    \begin{align*}
    \mu & = \dfrac{\sigma_{q^*}^2 + \mu_{q^*}^2 - \mu_p\mu_{q^*}}{\mu_p - \mu_{q^*}},\\
    \lambda &=  \dfrac{\sigma_{q^*}^2}{\mu_p - \mu_{q^*}}.
    \end{align*}
    
    Substituting them into the first equation in the original system Eq. \eqref{eqn:zero-partial-a}, and recalling the relationship between $\lambda$ and $\mu$ from Eq. \eqref{eqn:UCB-lambda-mu-system-A}, we get that for each x:
    \begin{align}
    \nonumber \dfrac{p_x}{q^*_x} &= \dfrac{v_x}{\lambda} + \dfrac{\mu}{\lambda} \\
    \nonumber  &= \dfrac{v_x}{\lambda} + \dfrac{\mu}{\mu_{q^*} + \mu} \\
    \nonumber  &= \dfrac{v_x}{\lambda} + \dfrac{\mu_{q^*} + \mu - \mu_{q^*}}{\mu_{q^*} + \mu} \\
    &=1 + \dfrac{v_x - \mu_{q^*}}{\lambda} \label{eqn:UCB-ratio-sigma}.
    \end{align}
    
    We can now rewrite the optimization problem in Eq. \eqref{eqn:UCB-originalOPT} in terms of our new variables using Eq. \eqref{eqn:UCB-ratio-sigma}.
    
    The positivity constraint in Eq. \eqref{eqn:UCB-align-opt-greater-than-zero-constraint} and recalling that $p_x > 0$ for all $x \in S$, yields,
    \begin{equation*}
    \dfrac{p_x}{q^*_x} = 1 + \dfrac{v_x - \mu_{q^*} }{\lambda} > 0,
    \end{equation*}
    
    the normalization constraint in Eq. \eqref{eqn:UCB-align-opt-sum-to-one-constraint} yields,
    
    \begin{equation*}
    \sum_x \dfrac{p_x}{1 + \dfrac{v_x - \mu_{q^*}}{\lambda}} =1,
    \end{equation*}
    
    and the KL divergence constraint in Eq. \eqref{eqn:UCB-align-opt-delta-constraint} yields,
    
    \begin{equation*}
    \sum_{x \in S} p_x \ln \left( 1 + \dfrac{v_x - \mu_q^*}{\lambda} \right) = \delta.
    \end{equation*}
    
    Observe that $\mu_p$ must be strictly less than $\mu_q^*$. To see this, take $\underline{q} = \underline{p}$, then $\underline{q}$ is feasible and the left hand side of Eq. \eqref{eqn:UCB-align-original-opt-delta-constraint} is $0$ which is less than $\delta$. Lemma \ref{lem:UCB-kl-constraint-equality} implies there exists some feasible $\underline{q^\prime}$ with a strictly greater objective function, i.e. $\mu_p = \mu_q < \mu_q^\prime \leq \mu_q^*$. We also know that $\lambda < 0$ because $\sigma^2_{q^*} > 0$ by definition in Eq. \eqref{eqn:UCB-sigma-q-definition}.
    
    Thus we can rewrite the optimization problem in Eq. \eqref{eqn:UCB-originalOPT} as, follows:
    \begin{alignat}{2}
    \nonumber \text{max}_{\mu_q,\lambda} \quad   & \mu_q, &\\
    \nonumber \\
    \nonumber \text{s.t. } & \\
    \nonumber &\sum_{x \in S} p_x \ln \left( 1 + \dfrac{v_x - \mu_q}{\lambda} \right) = \delta, \\
    \nonumber & \sum_x p_x \dfrac{\lambda}{\lambda + v_x - \mu_q} =1, \\
    & 1 + \dfrac{v_x - \mu_q}{\lambda}  > 0 & \forall x \in S \label{eqn:UCB-last-constraint},\\
    \nonumber &\mu_p < \mu_q < V \text{ and }  \lambda < 0.
    \end{alignat}
    
    Having established that $\lambda$ is strictly less than zero we can simplify the last constraint, Eq. \eqref{eqn:UCB-last-constraint}, as follows. Let $V = \max_x v_x$
    
    \begin{equation*}
    \begin{split}
    1 + \dfrac{v_x - \mu_q}{\lambda}  & > 0, \; \forall x \in S \\
    \dfrac{v_x - \mu_q}{\lambda} & > -1, \; \forall x \in S \\
    v_x - \mu_q & < -\lambda, \; \forall x \in S \\
    \mu_q -v_x  & > \lambda, \; \forall x \in S \\
    \implies \mu_q -V  & >\lambda. \\
    \end{split}
    \end{equation*}
    
    Thus we have,
    
    \begin{alignat*}{2}
    \text{max}_{\mu_q,\lambda} \quad   & \mu_q, &\\
    \nonumber \\
    \nonumber \text{s.t. } & \\
    &\sum_{x \in S} p_x \ln \left( 1 + \dfrac{v_x - \mu_q}{\lambda} \right) = \delta, \\
    & \sum_x p_x \dfrac{\lambda}{\lambda + v_x - \mu_q} =1, \\
    &\mu_p < \mu_q < V \text{ and }  \lambda  < \mu_q -V .
    \end{alignat*}
    
    Which is just two equations with two unknowns. Recalling that any feasible solution will be a global maximum by our discussion of the convexity of the optimization problem, we have the desired result,
    \begin{equation*}
    C(\underline{p}, \underline{v}, \delta) = \mu_q^*,
    \end{equation*}
    Where the only unknowns are $\mu_q^*$ and $\lambda$, and they satisfy these constraints:
    \begin{align*}
    \sum_{x \in S} p_x \ln \left( 1 + \dfrac{v_x - \mu_q^*}{\lambda} \right) = \delta, \\
    \sum_x p_x \dfrac{\lambda}{\lambda + v_x - \mu_q^*} =1, \\
    \mu_p < \mu_q^* < V \text{ and }  \lambda < \mu_q^* -V .
    \end{align*}
\end{proof}

\subsection{Proof of Theorem \ref{theorem-DMED-Simplification-degenerate}}\label{sec:DMED-degenerate-Proof}
First we restate Theorem \ref{theorem-DMED-Simplification-degenerate}:

The value of $D(\underline{p}, \underline{v}, \rho)$ and by extension $D_t(a)$ can be easily found in the following cases:
\begin{itemize}
    \item If $\rho > V$ then the optimization problem, $D(\underline{p}, \underline{v}, \rho)$ is infeasible and we say $D(\underline{p}, \underline{v}, \rho) = \infty$ and $D_t(a) = -T_{x_t,a}(t)$.
    
    \item If $\rho \leq \mu_p$ then $D(\underline{p}, \underline{v}, \rho) = 0$ and we say $D_t(a) = \infty$.
    
    \item If $v_{x_1} \neq v_{x_2}$ for some  $x_1,x_2 \in S$ and $\rho = V$, then optimization problem $D(\underline{p}, \underline{v}, \rho)$ diverges to infinity and we say $D(\underline{p}, \underline{v}, \rho) = \infty$ and $D_t(a) = -T_{x_t,a}(t)$.
\end{itemize}

\begin{proof}
    For $\rho > V = \max_x v_x$, the optimization problem is infeasible because there is no feasible $\underline{q}$ that will have an average more than $V$ (i.e. $\sum_x q_x v_x \leq V$). In that case we take $D(\underline{p}, \underline{v}, \rho) = \infty$ and the corresponding DMED discrepancy index $D_t(a) = - T_{x_t,a}(t)$.
    
    For any $\rho \leq \mu_p$, i.e. less than or equal to the expected value under the current estimates, $D(\underline{p}, \underline{v}, \rho) = 0$ by simply taking $\underline{q}^* = \underline{p}$ and we take the corresponding DMED discrepancy index $D_t(a) = \infty$.
    
    If $v_{x_1} = v_{x_2}$ for all $x_1,x_2 \in S$ then $\mu_p = v_x = V$ and depending on the value of $\rho$ one of the previous two situations apply.
    
    If $v_{x_1} \neq v_{x_2}$ for some  $x_1,x_2 \in S$ and $\rho = V$ we have the following. Any feasible $\underline{q}$ such that $\sum_x q_x v_x = V$ must have $q_x = 0$ for some $x \in S$ such that $v_x < V$, in which case $\underline{q}$ falls outside of $\Theta$ - and it is in fact not feasible. We therefore take $D(\underline{p}, \underline{v}, \rho) = \infty$ and the corresponding DMED discrepancy index $D_t(a) = - T_{x_t,a}(t)$.
\end{proof}

\subsection{Proof of Theorem \ref{theorem-DMED-Simplification}}\label{sec:DMED Simplification Proof}
In this section we will prove Theorem \ref{theorem-DMED-Simplification}, which we restate here. Let $V = \max_x v_x$. Then, for any $\underline{v}$ such that $v_{x_1} \neq v_{x_2}$ for some  $x_1,x_2 \in S$ and for $\sum_{x \in S} p_x v_x < \rho < V$,
\begin{equation*}
D(\underline{p}, \underline{v}, \rho) = \sum_x p_x \ln( 1 + (\rho - v_x)\lambda),
\end{equation*}
where
\begin{align*}
\sum_x p_x \frac{ \rho - v_x }{1 + (\rho - v_x)\lambda} = 0, \\
0 < \lambda < \frac{1}{V - \rho}.
\end{align*}

Before giving the formal proof, it may be helpful to understand the overall conception of the proof. The main idea is the use of Lagrange multiplier techniques, which greatly reduces the dimensionality of the problem to be solved. We are able to exchange from trying to find the optimal probability vector $\underline{q}^*$, to a problem where we need only find two moments of the optimal $\underline{q}^*$, a dramatic dimension reduction. In the MDP-DMED case we are able to simplify even further, because the optimal unknown mean $\mu_q^*$ is given as $\rho$, and it suffices to find $\lambda = (\mu_q^* - \mu_p)/\sigma_{q^*}^2$ which is a function of the unknown optimal variance.

The proof follows along similar lines as the one for MDP-UCB in Appendix \ref{sec:UCB Simplification Proof}.

\begin{proof}
    Recall that,
    \begin{equation}\label{eqn:DMED-original-OPT}
    \begin{split}
    D(\underline{p}, \underline{v}, \rho) & = \inf_{\underline{q} \in \Theta} \left\{ \mathbf{I}(\underline{p}, \underline{q}): \sum_x q_x v_x \geq \rho \right\}.
    \end{split}
    \end{equation}
    We want to show that the infimum in EQ. \eqref{eqn:DMED-original-OPT} is realized by a minimum.
    
    Let $ 0 < \epsilon <1$ and  $x^* = \arg\max v_x$. Consider the probability vector  $\underline{q}^\prime$ defined as $q^\prime_{x^*} = 1 - \epsilon$ and $q^\prime_{x} = \epsilon / \vert S \vert$ for $x \neq x^*$. For the appropriate choice of $\epsilon$, we will have $\sum_x q^\prime_x v_x = \rho < V$ with finite valued $\mathbf{I}(\underline{p}, \underline{q}^\prime)$. Thus, $D(\underline{p}, \underline{v}, \rho) \leq \mathbf{I}( \underline{p}, \underline{q}^\prime)$ and we can restrict to only considering $\underline{q} \in \Theta$ such that $\mathbf{I}(\underline{p}, \underline{q}) \leq \mathbf{I}(\underline{p}, \underline{q}^\prime)$. This feasible set is closed and compact, and hence the infimum is realized by a minimum over this set. Since $\mathbf{I}(\underline{p}, \underline{q}^\prime)$ is diverging to infinity as $\epsilon \to 0$, this minimum must occur in the interior of the constrained feasible region. Hence the infimum \textit{without} the additional constraint on feasibility will also be realized by a minimum within the interior of the set $\{ \underline{q} \in \Theta,  \sum_x q_x v_x \geq \rho \}$. 
    
    Thus, we can rewrite the problem of computing $D(\underline{p}, \underline{v}, \rho)$ in the following form:
    
    \begin{alignat}{2}
    \nonumber \text{min}_{\underline{q}} \quad   &  \sum_{x \in S} p_x \ln \dfrac{p_x}{q_x},   & \\
    \nonumber \\
    \nonumber \text{s.t. } & \\
    &  \quad  \sum_{x \in S} q_x v_x \geq \rho, \label{eqn:DMED-align-original-opt-rho-constraint} \\
    \nonumber &  \quad \sum_{x \in S} q_x = 1, \\
    \nonumber &  \quad q_x > 0 & x \in S.
    \end{alignat}
    
    Here we can use Lemma \ref{lem:DMED-rho-constraint-equality} in Appendix \ref{sec:KL-Lemmas} to observe that for any feasible $\underline{q}$ where the constraint in Eq. \eqref{eqn:DMED-align-original-opt-rho-constraint} is strict, we can construct a feasible $\underline{q}^\prime$ with a strictly smaller objective function (KL divergence w.r.t. $\underline{p}$). As such, the optimum must occur when this constraint is satisfied with equality, and the optimization problem can be re-written as so:
    
    \begin{alignat}{2}
    \nonumber \text{min}_{\underline{q}} \quad   &  \sum_{x \in S} p_x \ln \dfrac{p_x}{q_x},   &\\
    \nonumber \\
    \nonumber \text{s.t. } & \\
    &  \quad  \sum_{x \in S} q_x v_x = \rho, \label{eqn:DMED-align-opt-rho-constraint} \\
    &  \quad \sum_{x \in S} q_x = 1, \label{eqn:DMED-align-opt-sum-to-one-constraint} \\
    &  \quad q_x > 0 & x \in S. \label{eqn:DMED-align-opt-positive-constraint}
    \end{alignat}
    
    We now turn to the main task, reducing the dimension of the optimization problem. Using Lagrange multipliers we have the following auxiliary equation,
    
    \begin{equation*}
    \begin{split}
    L(\underline{q},\lambda,\mu) & = -\sum_{x \in S} p_x \ln \dfrac{p_x}{q_x} + \lambda \left( \sum_{x \in S} q_x v_x - \rho \right) + \mu \left( \sum_{x \in S} q_x  - 1 \right).
    \end{split}
    \end{equation*}
    
    Note when using the Lagrange multipliers, we can safely ignore the positivity constraints in Eq. \eqref{eqn:DMED-align-opt-positive-constraint} because they are strict inequalities, thus inactive, and thus have a Lagrange multiplier of zero.

    Taking partial derivatives, we get,
    \begin{equation*}
    \begin{split}
    L^{\prime}_{q_x}(\underline{q},\lambda,\mu) & = \dfrac{p_x}{q_x} + \lambda v_x + \mu  \; \; \; ,\; \; \forall x \in S, \\
    L^{\prime}_{\lambda}(\underline{q},\lambda,\mu) & = \sum_{x \in S} q_x v_x - \rho, \\
    L^{\prime}_{\mu}(\underline{q},\lambda,\mu) & =  \sum_{x \in S} q_x  - 1.
    \end{split}
    \end{equation*}
    
    Setting them to zero, results in the following system of equations for the optimal solution, $\underline{q}^*$,
    \begin{align}
    -\dfrac{p_x}{q^*_x} &= \lambda v_x + \mu \; \; \; ,\; \; \forall x \in S \label{eqn:DMED-zero-partial-a}, \\
    \nonumber \sum_{x \in S} q^*_x v_x & = \rho, \\
    \nonumber \sum_{x \in S} q^*_x &= 1.
    \end{align}
    
    We are looking for a solution $\underline{q}^*$ to this system, and any such solution will be a global minimum. To see this, observe that our optimization problem is a convex optimization problem. We are minimizing a convex function, with affine equality constraints. Thus, any stationary point will be a local minimum, and any local minimum will be a global minimum. \citep{boyd2004convex}
    
    Consider the first equation: multiply through by $q^*_x$ to get $-p_x = \lambda v_x q^*_x + \mu q^*_x$. Summing this over $x$ and simplifying accordingly, we get $-1 = \lambda \rho + \mu$.
    
    If we take $-p_x = \lambda v_x q^*_x + \mu q^*_x$ and multiply through by $v_x$, we get $- v_x p_x = \lambda v^2_x q^*_x + \mu v_x q^*_x$. We now introduce two new quantities, $\rho_p$, the mean under transition law $\underline{p}$, and $\sigma^2_{q^*}$, the variance under transition law $\underline{q}^*$, explicitly defined as follows
    \begin{equation*}
    \begin{split}
    \rho_p & = \sum_x p_x v_x,
    \end{split}
    \end{equation*}
    
    \begin{equation} \label{eqn:sigma-q-definition}
    \sigma^2_{q^*} = \sum_x v^2_x q^*_x - \rho^2. \\
    \end{equation}
    
    Summing $- v_x p_x = \lambda v^2_x q^*_x + \mu v_x q^*_x$ over $x$ and simplifying accordingly, we get $-\rho_p = \lambda ( \sigma^2_{q^*} + \rho^2) + \mu \rho$. So we have two equations and two unknowns,
    \begin{equation*}
    \begin{split}
    -1 & = \lambda \rho + \mu, \\
    -\rho_p &= \lambda ( \sigma^2_{q^*} + \rho^2) + \mu \rho.
    \end{split}
    \end{equation*}
    
    Solving these for $\lambda$ and $\mu$  we have,
    \begin{equation} \label{eqn:DMED-lambda-mu-solution}
    \begin{split}
    \lambda & = \dfrac{\rho -\rho_p}{\sigma^2_{q^*}}, \\
    \mu & = -1 - \dfrac{\rho -\rho_p}{\sigma^2_{q^*}} \rho.
    \end{split}
    \end{equation}
    
    Substituting them into the first equation in the original system Eq. \eqref{eqn:DMED-zero-partial-a}, and noting that Eq. \eqref{eqn:DMED-lambda-mu-solution} implies $\mu = -1 - \lambda \rho$, we get that for each x:
    \begin{align}
    \nonumber \dfrac{p_x}{q^*_x} &= -\lambda v_x - \mu\\
    \nonumber  &= -\lambda v_x + 1 +\lambda \rho \\
    &= 1 + (\rho - v_x)\lambda  \label{eqn:DMED-ratio-sigma}.
    \end{align}
    
    In order to reduce the original problem to a 1-dimensional problem, we now express each of the constraints in terms of our new variables using Eq. \eqref{eqn:DMED-ratio-sigma}. The positivity constraint in Eq. \eqref{eqn:DMED-align-opt-positive-constraint} and recalling that $p_x > 0$ for all $x \in S$, yields,
    \begin{equation*}
    \frac{p_x}{q^*_x} = 1 + (\rho - v_x)\lambda > 0,
    \end{equation*}
    
    the normalization constraint in Eq. \eqref{eqn:DMED-align-opt-sum-to-one-constraint} yields,
    \begin{equation*}
    \sum_{x} \frac{p_x}{1 + (\rho - v_x)\lambda} = 1,
    \end{equation*}
    
    and the mean constraint in Eq. \eqref{eqn:DMED-align-opt-rho-constraint} yields,
    \begin{equation*}
    \sum_{x \in S} \frac{p_x}{1 + (\rho - v_x)\lambda} v_x = \rho.
    \end{equation*}
    
    Therefore, we can express the problem in Eq. \eqref{eqn:DMED-original-OPT}, noting Eq. \eqref{eqn:DMED-ratio-sigma} above for the $p_x / q^*_x$ term, as follows:
    
    \begin{alignat}{2}
    \nonumber \text{min}_{\lambda} \quad   &  \sum_{x} p_x \ln\left( 1 +  \lambda ( \rho - v_x) \right),   & \\
    \nonumber \\
    \nonumber \text{s.t. } & \\
    \nonumber  &\quad  \sum_{x} \frac{p_x}{ 1 + ( \rho - v_x) \lambda } = 1, \\
    \nonumber &\quad \sum_{x \in S} \frac{p_x}{1 + (\rho - v_x)\lambda} v_x = \rho, \\
    &  \quad 1 +  \lambda ( \rho - v_x) >  0 & \forall x \in S. \label{eqn:DMED-last-constraint}
    \end{alignat}
    
    We next establish feasible bounds for $\lambda$. Observe that the variance, $\sigma^2_{q^*}$ is strictly greater than 0 by definition in Eq. \eqref{eqn:sigma-q-definition} and by recalling that there exists some $x_1,x_2 \in S$ such that $v_{x_1} \neq v_{x_2}$. We also know that $\rho > \rho_p = \sum_x p_xv_x$ by assumption. Thus, $\lambda > 0$.
    
    Having established that $\lambda$ is strictly greater than zero we can simplify the last constraint, Eq. \eqref{eqn:DMED-last-constraint}, as follows. Let $V = \max_x v_x$,
    
    \begin{equation*}
    \begin{split}
    1 +  \lambda ( \rho - v_x) & > 0, \; \forall x \in S \\
    \implies 1 +  \lambda ( \rho - V) & > 0 \\
    1 +  \lambda\rho - \lambda V & > 0 \\
    1 +  \lambda\rho & > \lambda V \\
    1 & > \lambda (V - \rho) \\
    \dfrac{1}{(V - \rho)} & > \lambda.  \\
    \end{split}
    \end{equation*}
    Where the last step is justified by recalling that by assumption $V$  is strictly greater than $\rho$.
    
    So, $0 < \lambda < \dfrac{1}{(V - \rho)}$ and our optimization problem becomes,
    \begin{alignat}{2}
    \nonumber \text{min}_{\lambda} \quad   &  \sum_{x} p_x \ln\left( 1 +  \lambda ( \rho - v_x) \right),   & \\
    \nonumber \\
    \nonumber \text{s.t. } & \\
    &  \quad  \sum_{x} \frac{p_x}{ 1 + ( \rho - v_x) \lambda } = 1, \label{eqn:DMED-normalization-constraint} \\
   \nonumber & \quad \sum_{x \in S} \frac{p_x}{1 + (\rho - v_x)\lambda} v_x = \rho, \\
    \nonumber & \quad 0 < \lambda < \dfrac{1}{(V - \rho)}.
    \end{alignat}
    
    Taking a closer look at the normalization constraint, Eq. \eqref{eqn:DMED-normalization-constraint},
    \begin{equation*}
    \begin{split}
    0 &= \sum_{x} \frac{p_x}{ 1 +  \lambda ( \rho - v_x) } - 1 \\
    & = \sum_{x} p_x \left( \frac{1}{ 1 +  \lambda ( \rho - v_x) } - 1 \right) \\
    & = \sum_{x} p_x \left( \frac{1}{ 1 +  \lambda ( \rho - v_x) } - \dfrac{1 +  \lambda ( \rho - v_x)}{1 +  \lambda( \rho - v_x)} \right) \\
    & = \sum_{x} p_x \left( \frac{1 - 1 -  \lambda ( \rho - v_x)}{ 1 +  \lambda ( \rho - v_x) } \right) \\
    & = -\lambda  \sum_{x} p_x \left( \frac{( \rho - v_x)}{ 1 +  \lambda ( \rho - v_x) } \right).
    \end{split}
    \end{equation*}
    
    However, recalling that $\lambda$ is strictly positive, it must be that $\sum_{x} p_x \left( \frac{( \rho - v_x))}{ 1 +  \lambda ( \rho - v_x) } \right) = 0$. Hence we have:
    \begin{alignat}{2}
    \nonumber \text{min}_{\lambda} \quad   &  \sum_{x} p_x \ln\left( 1 +  \lambda ( \rho - v_x) \right),   & \\
    \nonumber \\
    \nonumber \text{s.t. } & \\
    &  \quad  \sum_{x} p_x \left( \frac{( \rho - v_x))}{ 1 +  \lambda ( \rho - v_x) } \right) = 0, \label{eqn:DMED-opt-0} \\
    & \quad \sum_{x \in S} \frac{p_x}{1 + (\rho - v_x)\lambda} v_x = \rho, \label{eqn:DMED-opt-rho} \\
    \nonumber & \quad 0 < \lambda < \dfrac{1}{(V - \rho)}.
    \end{alignat}
    
    Next we show that any $\lambda$ that satisfies Eq. \eqref{eqn:DMED-opt-0} will also satisfy Eq. \eqref{eqn:DMED-opt-rho} and thus we can remove that constraint,
    
    \begin{equation*}
    \begin{split}
    0 & = \sum_{x} p_x \left( \frac{( \rho - v_x))}{ 1 +  \lambda ( \rho - v_x) } \right) \\
    & =  \sum_{x} \dfrac{-p_x v_x}{{ 1 +  \lambda ( \rho - v_x) }}  + \sum_{x} \dfrac{p_x \rho}{{ 1 +  \lambda ( \rho - v_x) }} \\
    & = \sum_{x} \dfrac{-p_x v_x}{{ 1 +  \lambda ( \rho - v_x) }}  + \rho \sum_{x} \dfrac{p_x }{{ 1 +  \lambda ( \rho - v_x) }} \\
    & = \sum_{x} \dfrac{-p_x v_x}{{ 1 +  \lambda ( \rho - v_x) }}  + \rho \cdot 1.\\
    \end{split}
    \end{equation*}
    Where the last line is justified by recalling Eq. \eqref{eqn:DMED-normalization-constraint}. Thus we have established that,
    \begin{equation*}
    \begin{split}
    \sum_{x} \dfrac{-p_x v_x}{{ 1 +  \lambda ( \rho - v_x) }} = -\rho
    \implies  \sum_{x} \dfrac{p_x v_x}{{ 1 +  \lambda ( \rho - v_x) }} = \rho,
    \end{split}
    \end{equation*}
    which is Eq. \eqref{eqn:DMED-opt-rho}.
    
    Thus we can write the optimization problem as,
    \begin{alignat}{2}
    \nonumber \text{min}_{\lambda} \quad   &  \sum_{x} p_x \ln\left( 1 +  \lambda ( \rho - v_x) \right),   & \\
    \nonumber \\
    \nonumber \text{s.t. } & \\
    &  \quad  \sum_{x} p_x \left( \frac{( \rho - v_x))}{ 1 +  \lambda ( \rho - v_x) } \right) = 0, \label{eqn:DMED-first-constraint}\\
    \nonumber &  \quad 0 <  \lambda < \frac{1}{ V - \rho }.
    \end{alignat}
    
    Recall that any feasible solution will be a global minimum, by our discussion of the convexity of the optimization problem. To find a feasible solution, notice that the derivative of the objective function with respect to $\lambda$ is simply the first constraint, Eq. \eqref{eqn:DMED-first-constraint}. Therefore any stationary point of the objective function will satisfy the constraint, be feasible, and thus be a global minimum. Hence, we may replace the original optimization problem with the problem of solving,
    \begin{equation*}
    \sum_{x} p_x \left( \frac{( \rho - v_x))}{ 1 +  \lambda ( \rho - v_x) } \right) = 0,
    \end{equation*}
    subject to $ 0 < \lambda < \frac{1}{ V - \rho }$.
    
    Thus we have the desired result,
    \begin{equation*}
    D(\underline{p}, \underline{v}, \rho) = \sum_x p_x \ln( 1 + (\rho - v_x)\lambda),
    \end{equation*}
    Where the only unknown is $\lambda$, and it satisfies these constraints:
    \begin{align*}
    \sum_x p_x \frac{ \rho - v_x }{1 + (\rho - v_x)\lambda} = 0, \\
    0 < \lambda < \frac{1}{V - \rho}.
    \end{align*}
\end{proof}

\section{KL Divergence Optimization Lemmas} \label{sec:KL-Lemmas}
The purpose of this section is to state and prove a number of lemmas associated with convex optimization problems involving KL-Divergence terms. They are relevant, but tangential to most of the content of the paper.

In this section, we take $\underline{p} \in \Theta$ to be a distribution over $S$, with $\underline{v}$ to be the vector of intermediate state values. It is convenient to define $\mu_p = \sum_{x} p_x v_x$ and $V = \max_x v_x$. The vector $\underline{q}$ is taken to be another distribution over $S$, with possibly zero-valued elements. The KL Divergence between $\underline{p}$ and $\underline{q}$ is given by
\begin{equation*}
\textbf{I}(\underline{p}, \underline{q}) = \sum_x p_x \ln \frac{ p_x }{ q_x }.
\end{equation*}

\begin{lemma}\label{lem:UCB-kl-constraint-equality}
    Let $\underline{q} \in \Theta$ be such that $\textbf{I}(\underline{p}, \underline{q}) < \delta < \infty$, and suppose $v_{x_1} > v_{x_2}$ for some $x_1,x_2 \in S$. Then there is a valid probability distribution $\underline{q}^\prime$ such that $\textbf{I}(\underline{p}, \underline{q}^\prime) \leq \delta$, and
    \begin{equation*}
    \sum_{x \in S} q_x v_x  < \sum_{x \in S} q^\prime_x v_x.
    \end{equation*}
\end{lemma}

\begin{proof}
    Consider constructing an alternative $\underline{q}^\prime \in \Theta$ in the following way. Define $q^\prime_{x_1} = q_{x_1} + \Delta$, $q^\prime_{x_2} = q_{x_2} - \Delta$, and $q^\prime_x = q_x$ for $x \neq x_1,x_2$. Note that for $0 \leq \Delta < \min(q_{x_1}, q_{x_2})$, $\underline{q}^\prime$ will be a valid probability distribution vector over $S$.
    
    We have that for $\Delta > 0$,
    \begin{equation*}
    \begin{split}
    \sum_{x} q^\prime_{x} v_{x} - \sum_{x} q_{x} v_{x} & = (q_{x_1} + \Delta) v_{x_1} + (q_{x_2} - \Delta) v_{x_2} - q_{x_1} v_{x_1} - q_{x_2} v_{x_2} \\
    & = \Delta ( v_{x_1} - v_{x_2} ) \\
    & > 0.
    \end{split}
    \end{equation*}
    
    It remains to show that the KL Divergence $\textbf{I}(\underline{p}, \underline{q}^\prime)$ does not exceed $\delta$. Note the following relations,
    \begin{equation*}
    \begin{split}
    \textbf{I}(\underline{p}, \underline{q}^\prime) & = \sum_{x} p_x \ln \frac{ p_x }{ q^\prime_x } \\
    & = \sum_{x \neq x_1, x_2} p_x \ln \frac{ p_x }{ q_x } + p_{x_1} \ln \frac{ p_{x_1} }{ q_{x_1} + \Delta } + p_{x_2} \ln \frac{ p_{x_2} }{ q_{x_2} - \Delta } \\
    & = \sum_{x} p_x \ln \frac{ p_x }{ q_x } + p_{x_1} \ln \frac{ p_{x_1} }{ q_{x_1} + \Delta } - p_{x_1} \ln \frac{ p_{x_1} }{ q_{x_1} } + p_{x_2} \ln \frac{ p_{x_2} }{ q_{x_2} - \Delta } - p_{x_2} \ln \frac{ p_{x_2} }{ q_{x_2} - \Delta } \\
    & = \textbf{I}( \underline{p}, \underline{q} ) + p_{x_1} \ln \frac{ q_{x_1} }{ q_{x_1} + \Delta } + p_{x_2} \ln \frac{ q_{x_2} }{ q_{x_2} - \Delta }.
    \end{split}
    \end{equation*}
    So, if $\Delta = 0$ then $\textbf{I}( \underline{p}, \underline{q}^\prime) = \textbf{I}( \underline{p}, \underline{q}) < \delta$. Noting that additional terms in the last equation above are smooth functions of $\Delta$, $\textbf{I}(\underline{p}, \underline{q}^\prime)$ will not exceed $\delta$ in a neighborhood of $\Delta = 0$. Thus for sufficiently small $\Delta > 0$, the Lemma holds. 
\end{proof}

\begin{lemma}\label{lem:DMED-rho-constraint-equality}
    For any $\underline{q}$ such that
    \begin{equation}
    \sum_{x \in S} q_x v_x > \rho \geq \sum_{x \in S} p_x v_x,
    \end{equation}
    if $v_{x_1} \neq v_{x_2}$ for some $x_1,x_2 \in S$, there exist distributions $\underline{q}^\prime$ such that $\textbf{I}(\underline{p}, \underline{q}^\prime) \leq \textbf{I}(\underline{p}, \underline{q})$ and
    \begin{equation*}
    \sum_{x \in S} q_x v_x > \sum_{x \in S} q^\prime_x v_x \geq \rho.
    \end{equation*}
\end{lemma}

\begin{proof}
    As a consequence of our assumption that $\sum_x q_x v_x > \sum_x p_x v_x$, there must be some $v_{x_1} \neq v_{x_2}$ such that $\underline{q}$ puts more weight on the larger and $\underline{p}$ puts more weight on the smaller. Let $v_{x_1} > v_{x_2}$, with $q_{x_1} > p_{x_1}$ and $q_{x_2} < p_{x_2}$.
    
    Consider constructing an alternative distribution $\underline{q}^\prime \in \Theta$ in the following way. For $0 \leq \Delta < q_{x_1}$, define $\underline{q}^\prime$ by $q^\prime_{x_1} = q_{x_1} - \Delta$, $q^\prime_{x_2} = q_{x_2} + \Delta$, and $q^\prime_x = q_x$ for $x \neq x_1,x_2$. As before, for $\Delta$ in this range, $\underline{q}^\prime \in \Theta$ represents a valid probability distribution on $S$.
    
    As in the proof of Lemma \ref{lem:UCB-kl-constraint-equality}, we have that for $\Delta > 0$,
    \begin{equation*}
    \begin{split}
    \sum_{x} q^\prime_{x} v_{x} - \sum_{x} q_{x} v_{x} & = (q_{x_1} + \Delta) v_{x_1} + (q_{x_2} - \Delta) v_{x_2} - q_{x_1} v_{x_1} - q_{x_2} v_{x_2} \\
    & = \Delta ( v_{x_1} - v_{x_2} ) \\
    & > 0.
    \end{split}
    \end{equation*}
    Taking $\Delta$ sufficiently small (so that the mean does not drop below $\rho$), we have that
    \begin{equation*}
    \sum_{x \in S} q_x v_x > \sum_{x \in S} q^\prime_x v_x \geq \rho.
    \end{equation*}
    
    It remains to show that $\textbf{I}(\underline{p}, \underline{q}^\prime) \leq \textbf{I}(\underline{p}, \underline{q})$. Similar to the proof of Lemma \ref{lem:UCB-kl-constraint-equality}, we have that
    \begin{equation*}
    \begin{split}
    \textbf{I}(\underline{p}, \underline{q}^\prime) & = \textbf{I}( \underline{p}, \underline{q} ) + p_{x_1} \ln \frac{ q_{x_1} }{ q_{x_1} - \Delta } + p_{x_2} \ln \frac{ q_{x_2} }{ q_{x_2} + \Delta }.
    \end{split}
    \end{equation*}
    Hence we see that $\textbf{I}(\underline{p}, \underline{q}^\prime)  = \textbf{I}( \underline{p}, \underline{q} )$ when $\Delta = 0$. Looking at the derivative of $\textbf{I}(\underline{p}, \underline{q}^\prime)$ with respect to $\Delta$ at $\Delta = 0$, we see
    \begin{equation*}
    \frac{d}{d\Delta} \textbf{I}( \underline{p}, \underline{q}^\prime) \vert_{\Delta = 0} = \frac{p_{x_1}}{q_{x_1}} - \frac{p_{x_2}}{q_{x_2}} < 0,
    \end{equation*}
    
    where the last step follows since $p_{x_1} / q_{x_1} < 1$ and $p_{x_2} / q_{x_2} > 1$, as discussed initially. Hence while the KL divergences are equal for $\Delta = 0$, $\textbf{I}(\underline{p}, \underline{q}^\prime)$ is decreasing within some small neighborhood, and the KL divergence between $\underline{p}$ and $\underline{q}^\prime$ is reduced.
\end{proof}


\end{document}